%% file: main.tex
\documentclass{article}

    \PassOptionsToPackage{numbers, compress}{natbib}

\usepackage[main, final]{neurips_2026}

\usepackage[utf8]{inputenc} 
\usepackage[T1]{fontenc}    
\usepackage{hyperref}       
\usepackage{url}            
\usepackage{booktabs}       
\usepackage{amsfonts}       
\usepackage{nicefrac}       
\usepackage{microtype}      
\usepackage{xcolor}         
\usepackage{graphicx}       
\input{macros}

\newcommand{\codeurl}{\url{https://github.com/google-deepmind/egg}}

\newif\ifshowacks
\showacksfalse  

\title{Does This Gradient Spark Joy?}

\author{%
  Ian Osband \\
  Google DeepMind \\
  \texttt{iosband@google.com} \\
}

\begin{document}

\maketitle
\begin{abstract}
Policy gradient computes a backward pass for every sample, although the backward pass is expensive and most samples carry little learning value.
The Delightful Policy Gradient (DG) scores learning value in the forward pass with \emph{delight}, the product of advantage and surprisal (negative log-probability).
We introduce the \emph{Kondo gate}, which compares delight against a compute price and pays for a backward pass only when the sample is worth it; sweeping the price traces a quality--cost trade-off.
On MNIST, keeping 3\% of gradient terms learns faster than standard policy gradient even when backward passes are free, and nearly matches full DG.
In bandits, zero-price gating preserves the expected gradient direction and removes its perpendicular noise, delight ranks samples more reliably than additive mixes of value and surprise, and a rare action with high reward variance marks where delight fails.
On transformer token reversal, the gate's advantage grows with sequence length and vocabulary size.
Because delight tolerates moderate noise, a cheap approximate forward pass could screen samples before expensive backpropagation, a training analogue of speculative decoding.
\end{abstract}

\section{Introduction}
\label{sec:intro}

Policy gradient methods compute a backward pass for every sample~\citep{williams1992simple}.
The backward pass is often substantially more expensive than the forward pass, yet most samples carry little learning value: a sample that confirms an action the policy already prefers, or punishes one it already avoids, contributes little to progress.
The Delightful Policy Gradient (DG) provides a forward-pass signal of learning value~\citep{osband2026delightfulpolicygradient}.
DG weights each gradient term by a sigmoid of \emph{delight}, the product of advantage and surprisal, so rare successes are emphasized and uninformative outcomes are suppressed.
Delight is available from the forward pass before any gradient computation.
DG improves over PG and PPO across staleness, actor bugs, reward corruption, and rare discovery~\citep{osband2026delightfulpolicygradient,osband2026delightfuldistributedpolicygradient}.
This suggests a sharper question: if the forward pass already identifies which samples are worth learning from, should we compute every backward pass at all?

We propose the \emph{Kondo gate}\footnote{Named for Marie Kondo's organizing principle: keep what sparks joy, discard what does not.}: keep the samples that spark joy, skip the rest.
For each sample, the learner compares delight against a compute price~$\lambda$ and draws a Bernoulli gate.
The probability of a backward pass increases with delight and decreases with price.
Sweeping the price traces a quality--cost trade-off; in practice, we set $\lambda$ adaptively to target a fraction $\rho$ of backward passes.
In distributed RL, where actors generate more trajectories than the learner can backpropagate through, Kondo can serve as admission control for the learner queue: cheap scoring decides which samples deserve learner time.

On MNIST at $\rho = 3\%$, the Kondo gate learns faster than standard PG per forward pass, before any backward savings are counted, and nearly matches full DG while using only $3\%$ of backward passes; in backward-pass space it wins by two orders of magnitude (Section~\ref{sec:mnist}).
In tabular bandits, we show that gating preserves useful gradient direction while eliminating perpendicular noise, and explain why delight is a better screening signal than additive combinations of value and surprise (Section~\ref{sec:bandit}).
The same analysis exposes the main limitation: a gambling regime in which high reward variance creates false delight on rare suboptimal actions (Section~\ref{sec:gambling}).
On transformer token reversal, the Kondo gate solves harder problems at equal backward compute, and its advantage grows with sequence length (Section~\ref{sec:results}).

\section{The Kondo Gate}
\label{sec:method}

Standard policy gradient uses per-sample updates $g_t = U_t \nabla_\theta \log \pi_\theta(A_t \mid \mathcal{H}_t)$, computing a backward pass for every sample regardless of learning value~\citep{williams1992simple}.
The Delightful Policy Gradient (DG) scores each sample by delight $\delight_t = U_t \cdot \surp_t$, the product of advantage and surprisal $\surp_t = -\log \pi_\theta(A_t \mid \mathcal{H}_t)$~\citep{osband2026delightfulpolicygradient}.
Delight is largest when a rare action succeeds, negative when a rare action fails, and small for actions the policy already expects.

DG uses delight to \emph{weight} gradient terms, but it still computes every backward pass.
The Kondo gate takes the next step: if delight already scores how much the learner can gain from a sample, then delight should also decide whether that sample deserves a backward pass at all.

\subsection{Implementation}
\label{sec:kondo_implementation}

Suppose each backward pass carries a price $\lambda \ge 0$.
For a single sample, we choose a gate probability $w \in [0,1]$ by maximizing
\begin{equation}
\label{eq:kondo_objective}
\max_{w \in [0,1]}\, \underbrace{\delight \, w}_{\text{learning value}} - \underbrace{\lambda \, w}_{\text{compute cost}} + \underbrace{\mix\, H(w)}_{\text{uncertainty}},
\end{equation}
where $H(w)$ is binary entropy and $\mix > 0$ is a temperature (derivation in Appendix~\ref{app:gate_derivation}).
The unique maximizer is $w^* = \sigma\!\big((\delight - \lambda) / \mix\big)$: the probability of paying for a backward pass increases with delight and decreases with price.
On a real computer the cost is all-or-nothing, so the Kondo gate samples $G_t \sim \mathrm{Ber}(w^*_t)$ and computes the backward pass only when $G_t = 1$.
Two limits anchor intuition: at $\mix \to 0$ the gate is a hard threshold $\mathbb{I}\{\delight > \lambda\}$, keeping only the most informative samples; at $\mix \to \infty$ the gate is constant, recovering standard PG up to a uniform rescaling.
In practice we set $\lambda$ adaptively to target a gate rate $\rho$, the fraction of samples that receive backward passes.

\algnewcommand{\algorithmichyperparameters}{\textbf{Hyperparameters:}}
\algnewcommand{\Hyperparameters}{\item[\algorithmichyperparameters]}
\begin{center}
\begin{minipage}{0.85\columnwidth}
\begin{algorithm}[H]
\caption{Delight-based Kondo gating for policy gradient}
\label{alg:kondo}
\begin{algorithmic}[1]
\Require Batch $\mathcal{B}$, policy $\pi_\theta$
\Hyperparameters gate rate $\rho \in (0,1]$ \textbf{or} price $\lambda$ and temperature $\mix > 0$
\For{$t \in \mathcal{B}$}
    \State $\surp_t \gets -\log \pi_\theta(A_t \mid \mathcal{H}_t)$ \Comment{Surprisal (forward pass)}
    \State $\delight_t \gets U_t \cdot \surp_t$ \Comment{Delight (forward pass)}
\EndFor
\If{gate rate $\rho$ given}
    \State $G_t \gets \mathbb{I}\{\delight_t \text{ is among the top } \lceil \rho |\mathcal{B}| \rceil\}$ for all $t$ \Comment{Fixed-rate gate}
\Else
    \State $G_t \sim \mathrm{Ber}\!\big(\sigma((\delight_t - \lambda) / \mix)\big)$ for all $t$ \Comment{Fixed-price gate}
\EndIf
\State $\Delta\theta \gets 0$
\For{$t \in \mathcal{B}$}
    \If{$G_t = 0$} \textbf{continue} \Comment{Skip backward pass} \EndIf
    \State $\Delta\theta \gets \Delta\theta + U_t \, \nabla_\theta \log \pi_\theta(A_t \mid \mathcal{H}_t)$ \Comment{Backward pass}
\EndFor
\State \Return $\Delta\theta$
\end{algorithmic}
\end{algorithm}
\end{minipage}
\end{center}

In this method, which we call DG-K, delight decides which terms to compute but does not reweight them.
Each retained term is the plain PG term, so relative to PG the gate only removes terms, and relative to DG it replaces continuous weighting with binary selection.
In expectation over the gate, $\E[\Delta\theta \mid \mathcal{B}] = \sum_{t \in \mathcal{B}} w^*_t\, U_t \nabla_\theta \log \pi_\theta(A_t \mid \mathcal{H}_t)$ is a DG-style weighted gradient.
Our fixed-rate experiments use the hard version, keeping the top-$\rho$ fraction of each batch by delight.
Both are deliberately biased relative to PG whenever the gate is selective.
On a single device this requires sparse or packed backpropagation; in distributed training the same rule is simpler: use $\delight_t$ as the priority for the learner queue and admit only the samples that exceed the current price.

\subsection{Why Delight, Not Simpler Priority Signals?}
\label{sec:kondo_why_delight}

A backward pass should be spent on samples that are both useful and non-redundant.
Advantage alone measures usefulness but does not identify how likely the sampled action was: two samples with equal advantage can have very different probabilities under the current policy, and the common one changes the policy little.
Surprisal alone measures rarity but ignores value: it prioritizes novelty for its own sake, including surprising failures that the learner has already learned to avoid.
Additive combinations $\alpha U + (1-\alpha) \surp$ interpolate between these two mistakes and require regime-dependent tuning of $\alpha$.

Because it multiplies advantage and surprisal, delight is large only when a sample is both valuable and unexpected under the current policy.
This makes it a natural screening signal for backward compute: keep the rare successes that teach the learner something new, skip the samples whose gradient is either redundant or actively unhelpful.
Section~\ref{sec:bandit} makes this precise in tabular bandits: it shows why delight is more reliable than additive alternatives and identifies the gambling regime in which it fails.

\section{MNIST Diagnostic}
\label{sec:mnist}

We begin with MNIST~\citep{lecun1998gradient}, the simplest neural-network RL problem: a contextual bandit with ten actions, immediate reward, and a two-layer MLP.
The setup follows \citet{osband2026delightfulpolicygradient}: an MLP with two hidden layers of 100 units (about 90K parameters) trained with Adam on batches of 100 images; only the gradient gating differs.
We parameterize the Kondo gate by a target \emph{gate rate} $\rho \in (0,1]$.
On each batch, we retain the top $\rho$ fraction of samples by delight.
Setting $\rho = 1$ recovers standard PG; small $\rho$ skips most backward passes.

\subsection{Core Results}
\label{sec:core_results}

Figure~\ref{fig:mnist} shows the main result at $\rho = 0.03$.
In forward-pass space~(a), the Kondo gate nearly matches full DG and both outperform PG, so gating preserves nearly all of the useful learning signal.
In backward-pass space~(b), the Kondo gate reaches the same error using two orders of magnitude fewer backward passes.
The green curve ends earlier, reaching nearly the same quality at a small fraction of the backward cost.

Why should a deliberately biased update learn faster than the unbiased policy gradient, even per forward pass?
Supervised learning makes the same choice: image classifiers are trained with cross-entropy, a surrogate that emphasizes hard examples, instead of the gradient of expected accuracy.
\citet{osband2026delightfulpolicygradient} shows that across contexts DG moves the expected update toward the cross-entropy direction; the Kondo gate keeps that reallocation while computing only $3\%$ of the terms.

\begin{figure}[ht!]
\centering
\begin{subfigure}[t]{0.48\columnwidth}
    \centering
    \includegraphics[width=\linewidth]{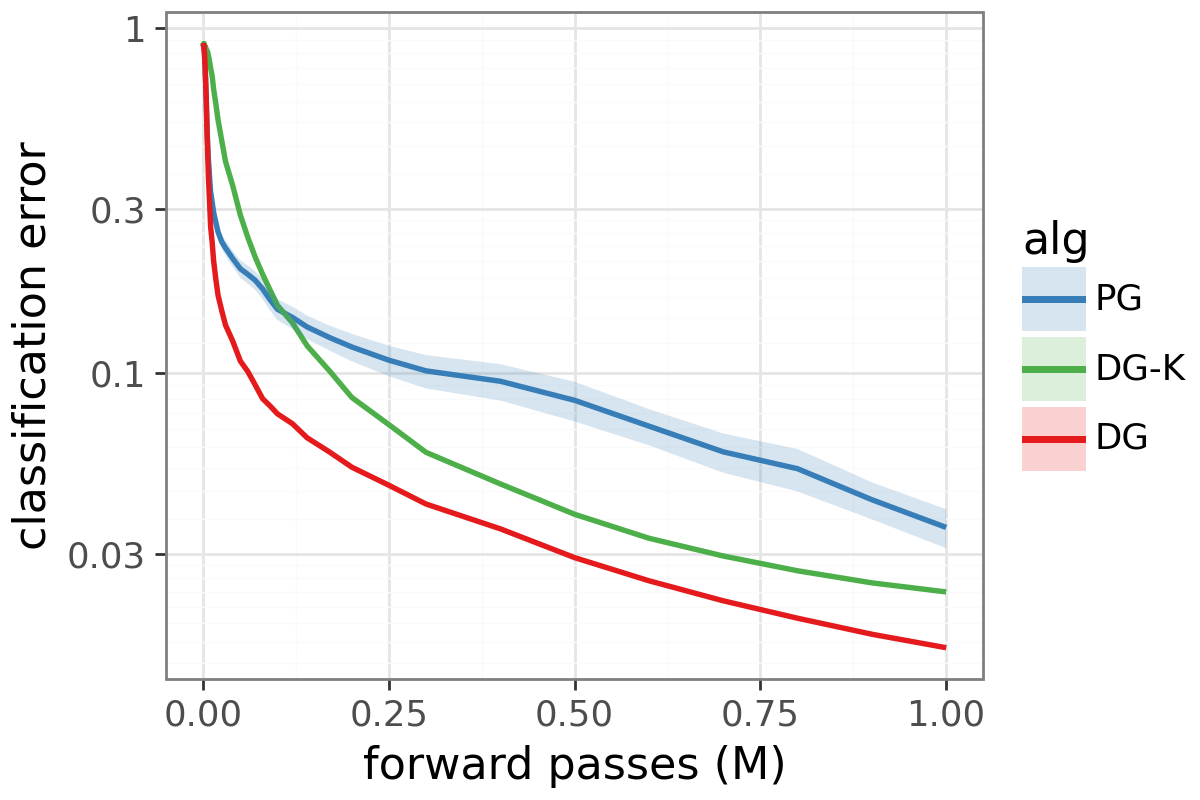}
    \caption{Forward passes: DG-K $\approx$ DG $\gg$ PG.}
    \label{fig:mnist_forward}
\end{subfigure}
\hfill
\begin{subfigure}[t]{0.48\columnwidth}
    \centering
    \includegraphics[width=\linewidth]{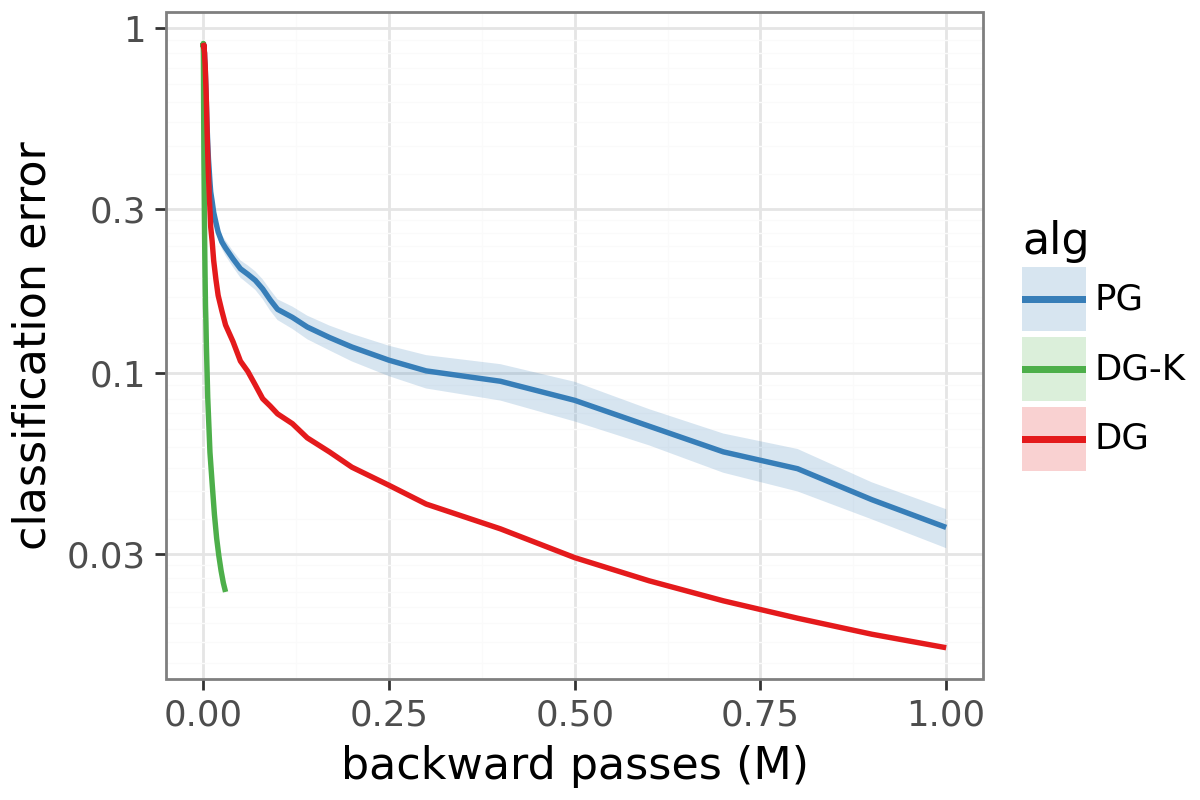}
    \caption{Backward passes: DG-K $\gg$ DG $\gg$ PG.}
    \label{fig:mnist_backward}
\end{subfigure}
\caption{PG, DG, and Kondo gate (DG-K) at $\rho = 0.03$ on MNIST.
(a)~The Kondo gate matches DG despite computing 3\% of backward passes.
(b)~It dominates by two orders of magnitude in backward-pass space.
Averaged over 30 seeds; shading shows $\pm 1$ standard error.}
\label{fig:mnist}
\end{figure}

Figure~\ref{fig:mnist_sweep} sweeps the gate rate $\rho$ from $0.01$ to $1.0$, with learning rate tuned per $\rho$.
In forward-pass space~(a), all gate rates converge to nearly the same final error ($\sim 0.5\%$): over this range, aggressive gating costs little in quality.
In backward-pass space~(b), $\rho = 0.01$ reaches low error with $\sim 100\times$ fewer backward passes than $\rho = 1.0$.

\begin{figure}[ht!]
\centering
\begin{subfigure}[t]{0.48\columnwidth}
    \centering
    \includegraphics[width=\linewidth]{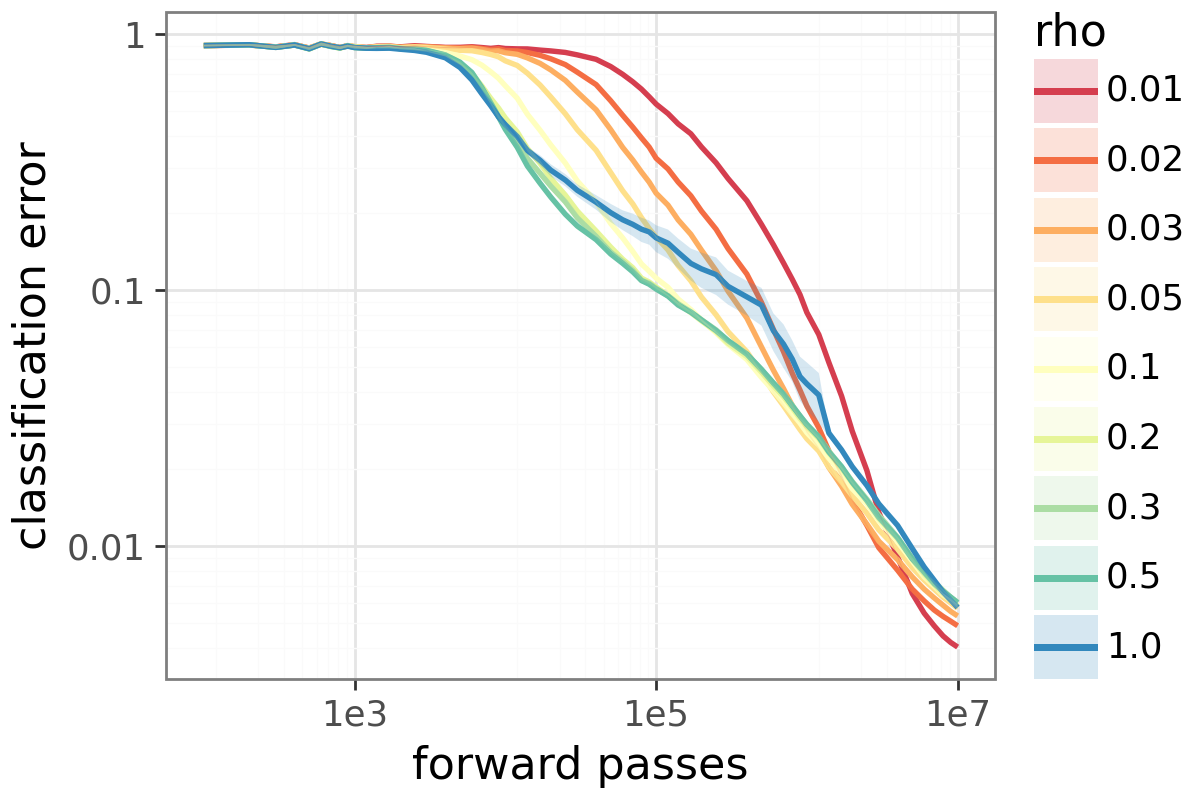}
    \caption{Forward steps: all $\rho$ converge to similar error.}
    \label{fig:sweep_forward}
\end{subfigure}
\hfill
\begin{subfigure}[t]{0.48\columnwidth}
    \centering
    \includegraphics[width=\linewidth]{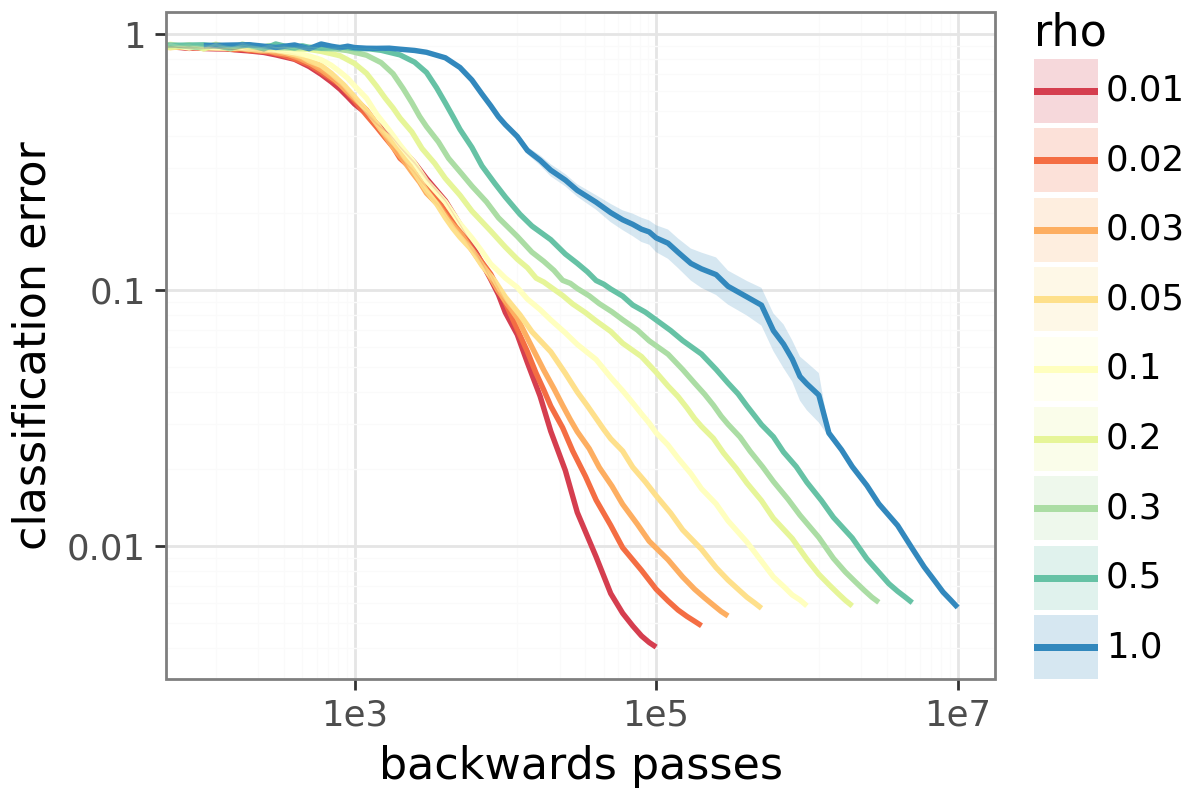}
    \caption{Backward steps: $100\times$ fewer at $\rho = 0.01$.}
    \label{fig:sweep_backward}
\end{subfigure}
\caption{Gate rate sweep ($\rho \in \{0.01, \ldots, 1.0\}$), learning rate tuned per $\rho$.
(a)~All gate rates converge to $\sim 0.5\%$ error eventually.
(b)~In backward-step space, smaller $\rho$ reaches any error with orders-of-magnitude fewer backward passes.}
\label{fig:mnist_sweep}
\end{figure}

Figure~\ref{fig:kept_main} shows which samples spark joy late in training.
The gate keeps digits that the policy classifies correctly but not yet with certainty, which we call the learning frontier.
It skips digits that are already solved, along with mistakes whose negative advantage makes delight negative.
The selection does not fade as training proceeds; at step $10{,}000$ kept samples still have systematically lower $\pi(y^*)$ than skipped ones (Appendix~\ref{app:mnist_histogram}).

\begin{figure}[ht!]
\centering
\includegraphics[width=0.9\columnwidth]{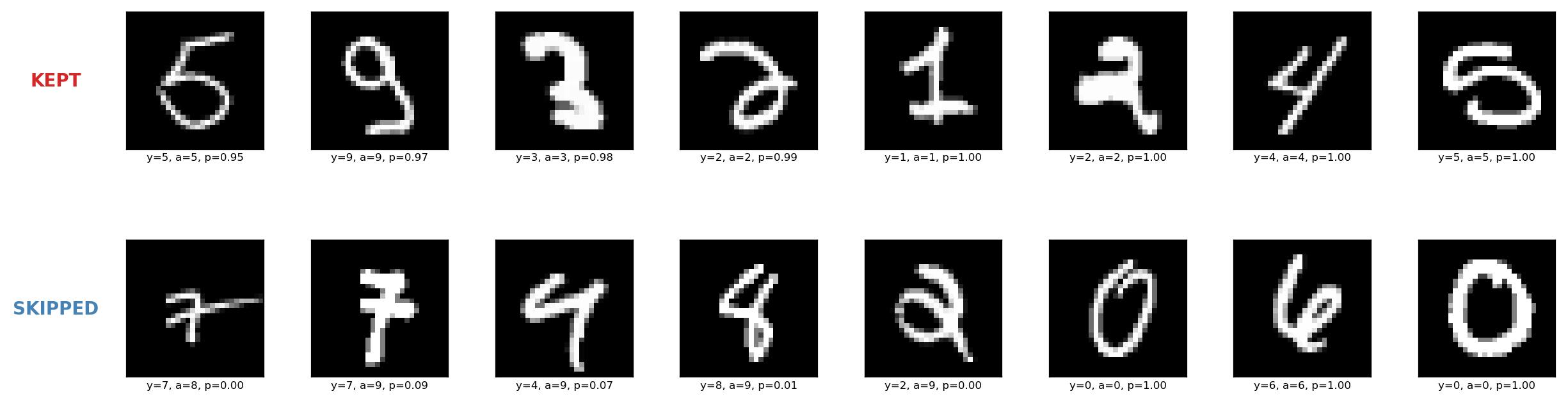}
\caption{Images kept (top) and skipped (bottom) by the Kondo gate at step $10{,}000$ ($\rho = 0.03$), annotated with label $y$, action $a$, and $p = \pi(y^*)$.
The gate keeps the learning frontier; it skips solved digits and confident mistakes.}
\label{fig:kept_main}
\end{figure}

\subsection{Compute Efficiency and Approximate Delight}
\label{sec:compute_efficiency}

The practical value of skipping backward passes depends on their cost.
Our experiments count selected gradient terms and do not measure wall-clock time.
Figure~\ref{fig:compute_efficiency} combines the forward and selected-term counts from the learning curves under an assumed backward/forward cost ratio, and reports the modeled compute (forward $+$ backward $\times$ ratio) to reach $5\%$ test error, normalized to PG.
This ratio is an assumed input: a backward pass costs about twice a forward pass in FLOPs~\citep{kaplan2020scaling}, and activation recomputation raises the practical ratio~\citep{narayanan2021efficient}.
Even at ratio $0$ (backward passes are free), DG-K improves over PG through better per-sample learning.
As the assumed ratio grows, the modeled speedup increases: at a ratio of $4\times$, DG-K needs $6\times$ less modeled compute than PG to reach the same error.
DG's speedup is constant ($\sim 2\times$) because it still computes every backward pass.

\begin{figure}[ht!]
\centering
\includegraphics[width=0.6\columnwidth]{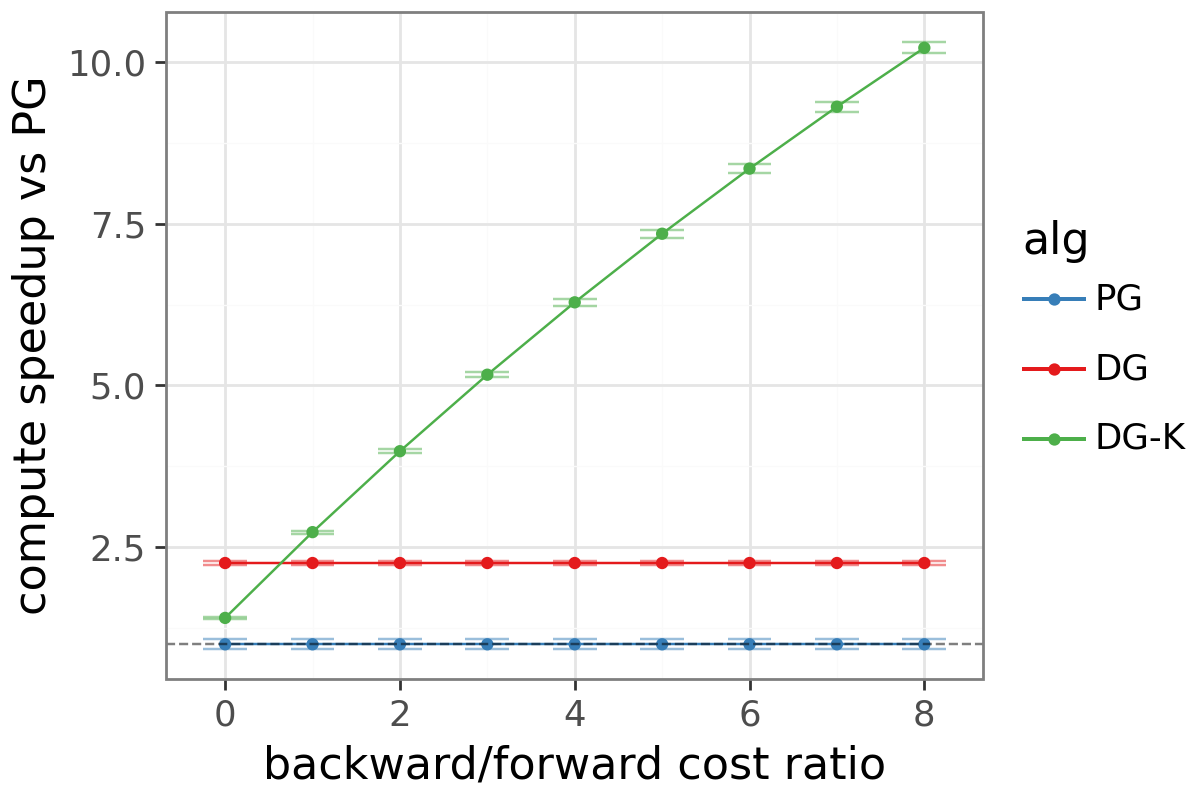}
\caption{Modeled compute speedup vs PG to reach 5\% test error on MNIST, as a function of an assumed backward/forward cost ratio.
DG's advantage is constant ($\sim 2\times$, better learning).
DG-K's advantage increases with backward cost (fewer backward passes).
At a ratio of $4\times$, the Kondo gate needs $6\times$ less modeled compute than PG.}
\label{fig:compute_efficiency}
\end{figure}

The gate's decision requires only a forward pass, which need not be computed at full precision.
If delight can be approximated, whether via quantized inference, a distilled model, or cached values, the \emph{effective} backward/forward cost ratio is much larger than $2\times$.
A cheap screening pass followed by an expensive full pass mirrors speculative decoding~\citep{leviathan2023fast}, applied to training.
Figure~\ref{fig:robust} tests this by injecting noise into the delight signal~(a) and the forward-pass logits~(b).
DG tolerates roughly $50\%$ relative delight noise and logit noise up to $\sigma_Z \approx 1$ before degrading; DG-K is more fragile in both cases.
These results suggest that screening need not be exact to capture much of the compute savings.

\begin{figure}[ht!]
\centering
\begin{subfigure}[t]{0.48\columnwidth}
    \centering
    \includegraphics[width=\linewidth]{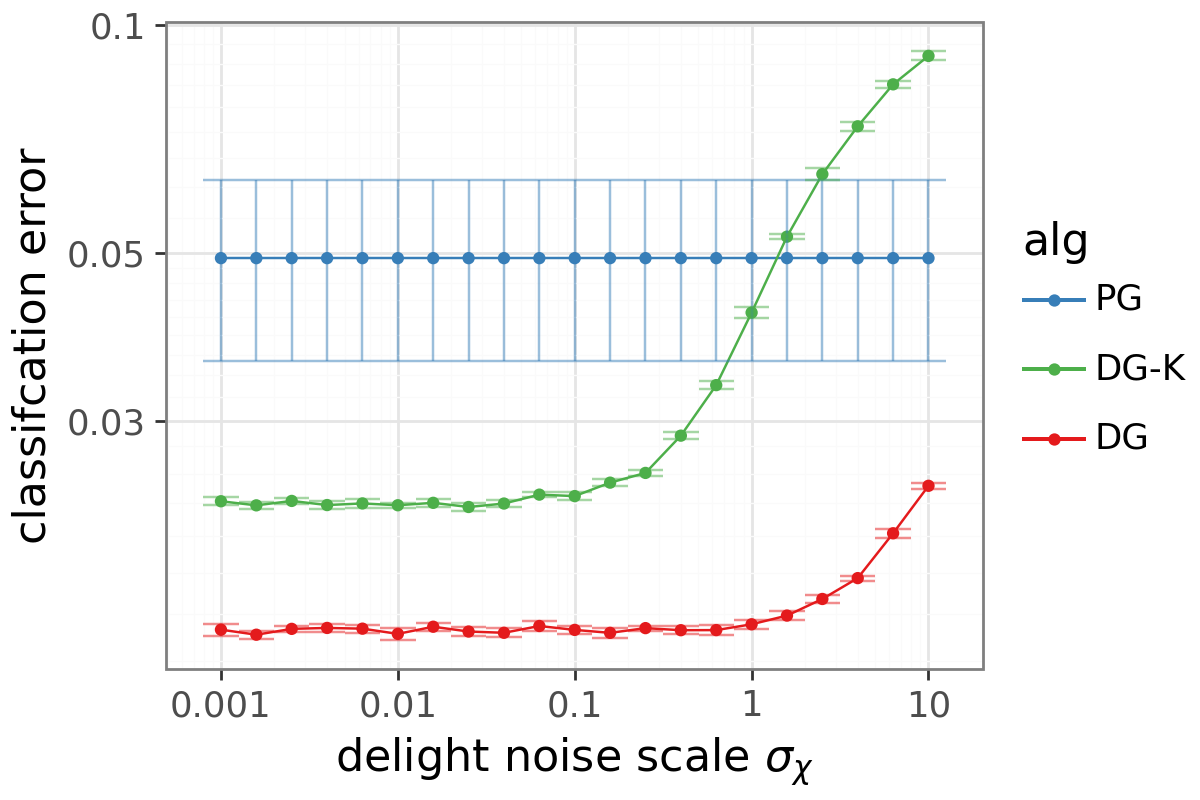}
    \caption{Delight noise (relative): DG tolerates $\sim 50\%$.}
    \label{fig:robust_delight}
\end{subfigure}
\hfill
\begin{subfigure}[t]{0.48\columnwidth}
    \centering
    \includegraphics[width=\linewidth]{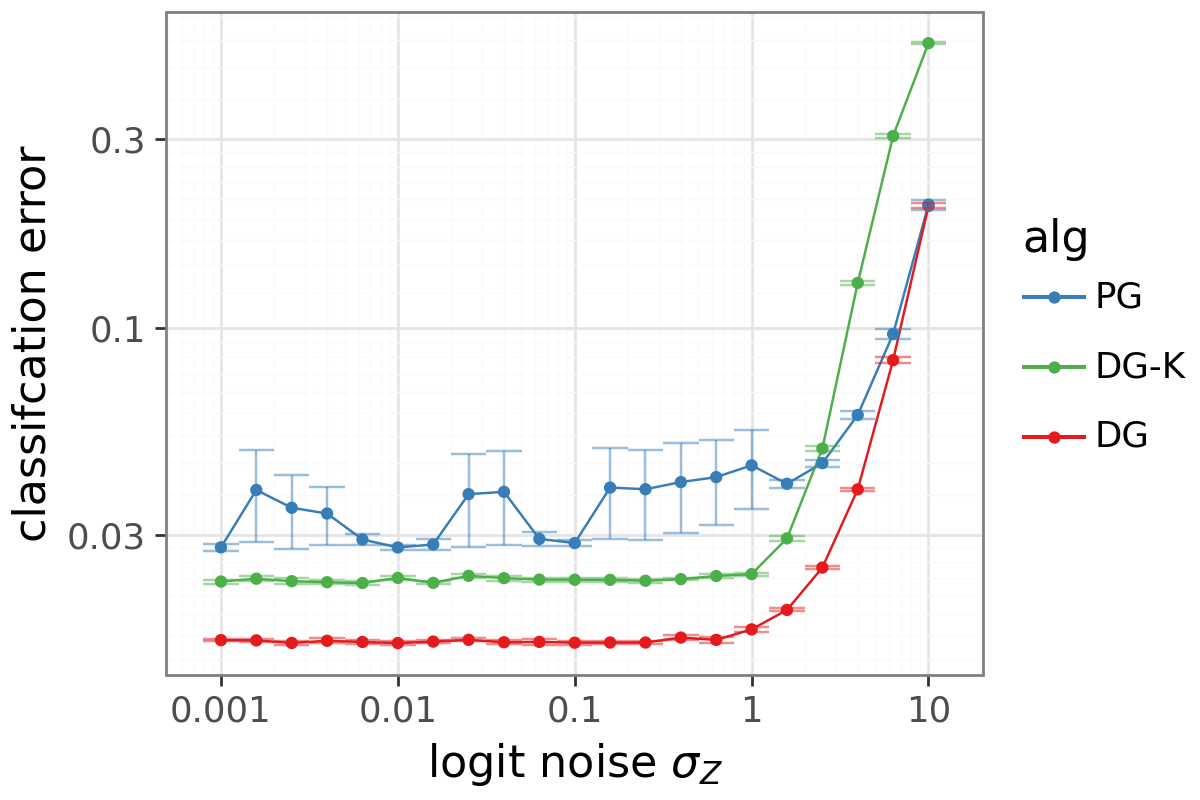}
    \caption{Logit noise: DG robust to $\sigma_Z \approx 1$.}
    \label{fig:robust_logit}
\end{subfigure}
\caption{Noise robustness on MNIST.
(a)~Delight noise scaled relative to $\mathrm{std}(\chi)$: DG tolerates $\sim 50\%$; DG-K degrades earlier.
(b)~Logit noise: DG is robust until $\sigma_Z \approx 1$; DG-K degrades faster.
Both validate that approximate forward passes and approximate delight preserve the gate's value.}
\label{fig:robust}
\end{figure}

The MNIST results show that most backward passes can be removed with little loss in learning quality, and that the resulting speedup grows with backward-pass cost; tabular bandits (Section~\ref{sec:bandit}) explain why.

\section{Tabular Analysis}
\label{sec:bandit}

We now answer the three questions raised by MNIST in a setting with exact gradients and no function-approximation error.
First, what useful part of the gradient survives the gate?
Second, why is delight the right priority signal rather than a simpler combination of value and surprise?
Third, when should delight-based screening fail?
Bandits make all three questions analytically tractable, while remaining close enough to MNIST that the predictions can be checked empirically.

\subsection{Pareto Improvement and Priority Signal}
\label{sec:theory_pareto}

The first question is what the gate keeps.
In a single-context softmax bandit, correct-action gradients are parallel to $\nabla J$ and carry no perpendicular noise.
Incorrect-action gradients contribute $\Theta(1)$ perpendicular component and only $\Theta(p)$ cosine with $\nabla J$.
With a baseline bounded away from zero, the batch cosine scales as $\Theta(p\sqrt{B})$ when $p^2 B \ll 1$: unless $B \gg 1/p^2$, the policy-gradient estimate is nearly random.
In this setting, zero-price gating is a Pareto improvement in gradient geometry: it preserves the useful signal, eliminates perpendicular noise, and reduces backward-pass cost.

\begin{proposition}[Kondo gate Pareto improvement]
\label{prop:pareto_formal}
Under a $K$-armed bandit with softmax policy $\pi = \mathrm{softmax}(z)$, deterministic reward $R = \mathbb{I}\{A = y^*\}$, and correct-action probability $p = \pi(y^*)$, consider the zero-price hard gate that keeps samples with $\delight > 0$ and skips those with $\delight < 0$:
\begin{enumerate}
\item Direction preserved: $\E[g_\mathrm{KG}] \propto \nabla_z J$.
\item Perpendicular variance eliminated: $\Var_\perp(g_\mathrm{KG}) = 0$.
\item Compute reduced: $pB$ backward passes instead of $B$.
\item Backward cost: for a baseline $b$ bounded away from zero, PG needs $\Omega(1/p^2)$ backward passes for $\cos = \Theta(1)$; KG achieves $\cos = 1$ with probability $\ge 1-\delta$ from $O(\log(1/\delta))$ backward passes, using about $\log(1/\delta)/p$ forward samples.
\end{enumerate}
\end{proposition}

This explains why aggressive gating can preserve learning quality.
The next question is why to threshold delight instead of a simpler additive score built from advantage and surprisal.
At $\lambda = 0$, gating on $\delight < 0$ is equivalent to gating on $U < 0$, so the product structure does not matter for sign alone.
It matters once we must rank \emph{among} positive-delight samples, as happens with $\lambda > 0$, multiple contexts, or approximate screening.

\begin{proposition}[Delight is sign-consistent; additive mixes can mis-rank]
\label{prop:delight_dominance}
Under Assumption~\ref{assump:bandit}, delight $\delight = U \cdot \surp$ satisfies $\delight(y^*) > 0 > \delight(a \neq y^*)$ for all $p, K$.
The additive family $f_\alpha = (1-\alpha) U + \alpha\surp$, with surprisal weight $\alpha \in [0,1)$, always ranks the correct action above incorrect ones when $p \le 1/K$; when $p > 1/K$ it does so only if $\alpha < \alpha^*(p,K) = \frac{1}{1+L}$, where $L = \log\!\big(\frac{p(K-1)}{1-p}\big) > 0$.
The admissible surprisal weight shrinks as policy quality or action-space size grows.
\end{proposition}

Multiplying the advantage by a positive surprisal preserves its sign, whereas adding surprisal can flip it.
Figure~\ref{fig:mnist_priority} checks this prediction on MNIST.
Delight remains robust across backward batch sizes and across the full sweep of additive-mix coefficients, whereas surprisal-only and additive priorities degrade or collapse.

\begin{figure}[ht!]
\centering
\begin{subfigure}[t]{0.48\columnwidth}
    \centering
    \includegraphics[width=\linewidth]{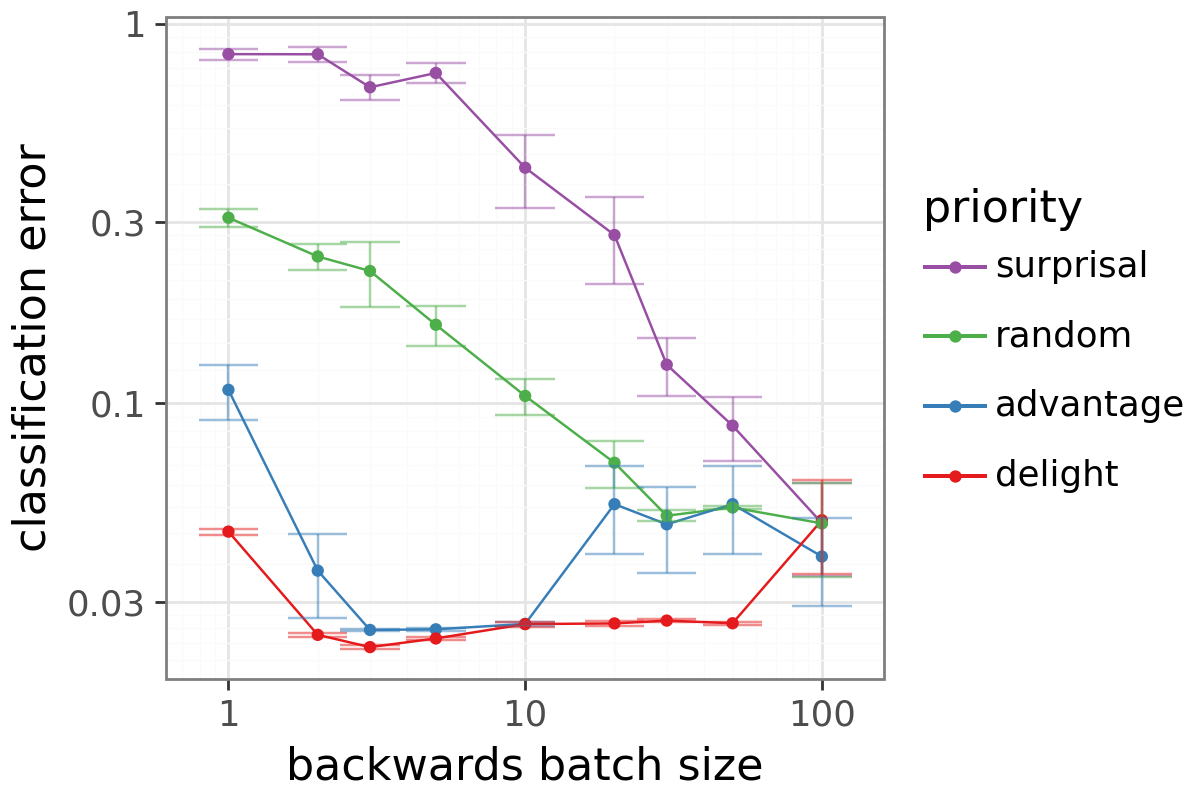}
    \caption{Error vs.\ backward batch size by priority.}
    \label{fig:priority_batch}
\end{subfigure}
\hfill
\begin{subfigure}[t]{0.48\columnwidth}
    \centering
    \includegraphics[width=\linewidth]{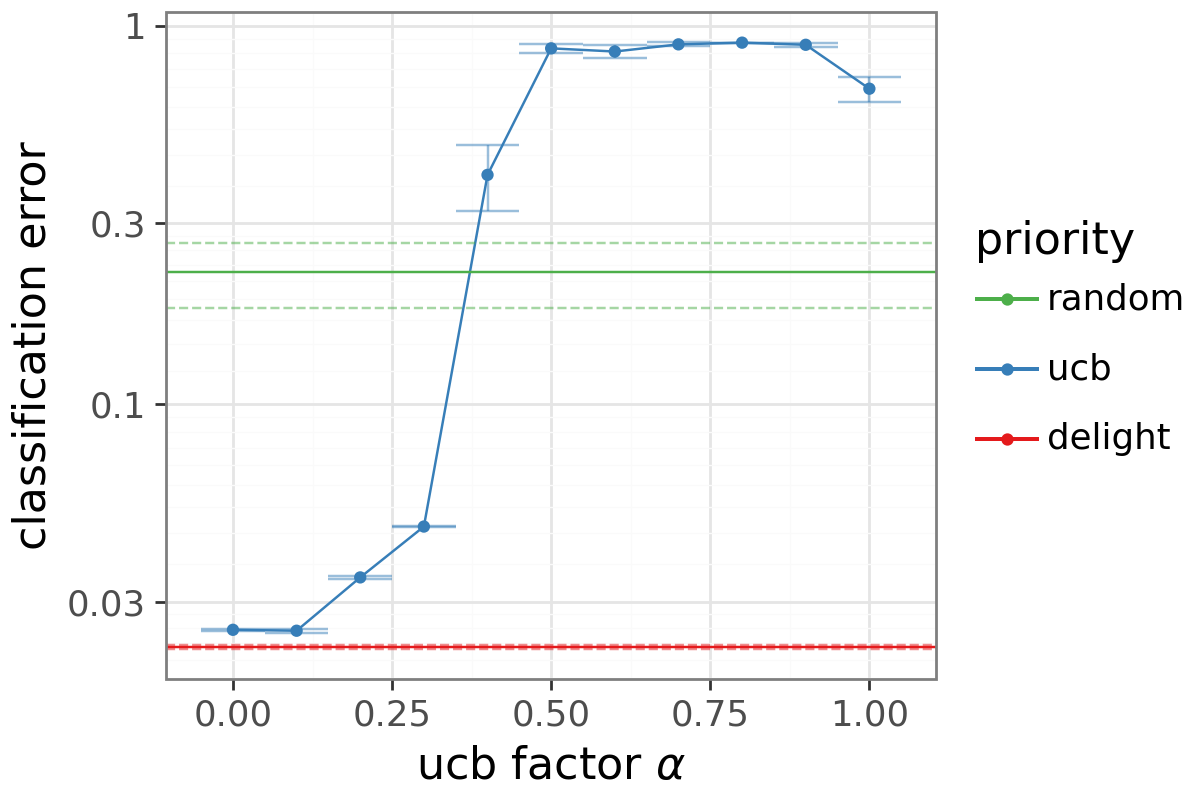}
    \caption{Error vs.\ surprisal weight $\alpha$; delight is flat.}
    \label{fig:priority_ucb}
\end{subfigure}
\caption{Priority signal comparison on MNIST.
(a)~Delight is robust across backward batch sizes; surprisal-only fails.
(b)~The additive mix collapses for $\alpha > 0.3$; delight (product) is $\alpha$-independent.
Consistent with Proposition~\ref{prop:delight_dominance}.}
\label{fig:mnist_priority}
\end{figure}

\subsection{The Gambling Pathology}
\label{sec:gambling}

The final question is when delight-based screening should fail.
The failure mode is a specific gambling regime in which a rare suboptimal action has such high reward variance that lucky draws masquerade as breakthroughs.

Consider a slot machine: arm~1 pays $\$1$ always; arm~2 pays $\$0$ with probability $0.99$ and $\$50$ with probability $0.01$.
Gap $\Delta = 0.50$, noise $\sigma \approx 5$, ratio $\sigma/\Delta = 10$.
When the slot machine hits, the learner observes $U > 0$, and because the policy rarely pulls arm~2, the high surprisal opens the gate wide.

\begin{proposition}[Gambling pathology]
\label{prop:gambling}
Two-armed bandit: arm~1 (optimal) with deterministic reward $\mu^*$, arm~2 (suboptimal) with Gaussian reward $R_2 \sim \mathcal{N}(\mu^* - \Delta, \sigma^2)$.
\begin{enumerate}
\item When $\sigma/\Delta \ll 1$: $\Pr(U_2 > 0 \mid A{=}2) \le \exp(-\Omega(\Delta^2/\sigma^2))$.
\item When $\sigma/\Delta \gg 1$: $\Pr(U_2 > 0 \mid A{=}2) = \Theta(1)$.
\item Delight amplifies: $|\delight_2| = |U_2| \cdot \log(1/\varepsilon)$ grows as the policy avoids arm~2.
\end{enumerate}
\end{proposition}

No per-sample statistic computed from $(R, \pi)$ can distinguish a genuine breakthrough from a lucky draw.
The pathology is driven by \emph{differential} $\sigma_a/\Delta_a$; in our homoskedastic experiment, DG and PG degrade together.
The same joint signal that makes delight valuable in normal learning becomes pathological in this regime, because a rare lucky draw on a suboptimal arm looks like a breakthrough.

We validate this on the MNIST bandit from Section~\ref{sec:mnist}.
To inject differential noise, we designate action $a{=}0$ as the ``gamble'': whenever the agent predicts~$0$ (regardless of true label), its reward receives additive $\mathcal{N}(0, \sigma_G^2)$ noise.
Figure~\ref{fig:bandit} sweeps two noise regimes.
Under homoskedastic noise $\sigma_R$, DG and PG degrade smoothly together~(a).
Under gambling noise $\sigma_G$ on a single action, DG outperforms PG for $\sigma_G < 1$ but collapses sharply near $\sigma_G \approx 1$ while PG degrades gracefully~(b), consistent with the regimes of Proposition~\ref{prop:gambling}.

\begin{figure}[ht!]
\vspace{-2mm}
\centering
\begin{subfigure}[t]{0.48\columnwidth}
    \centering
    \includegraphics[width=\linewidth]{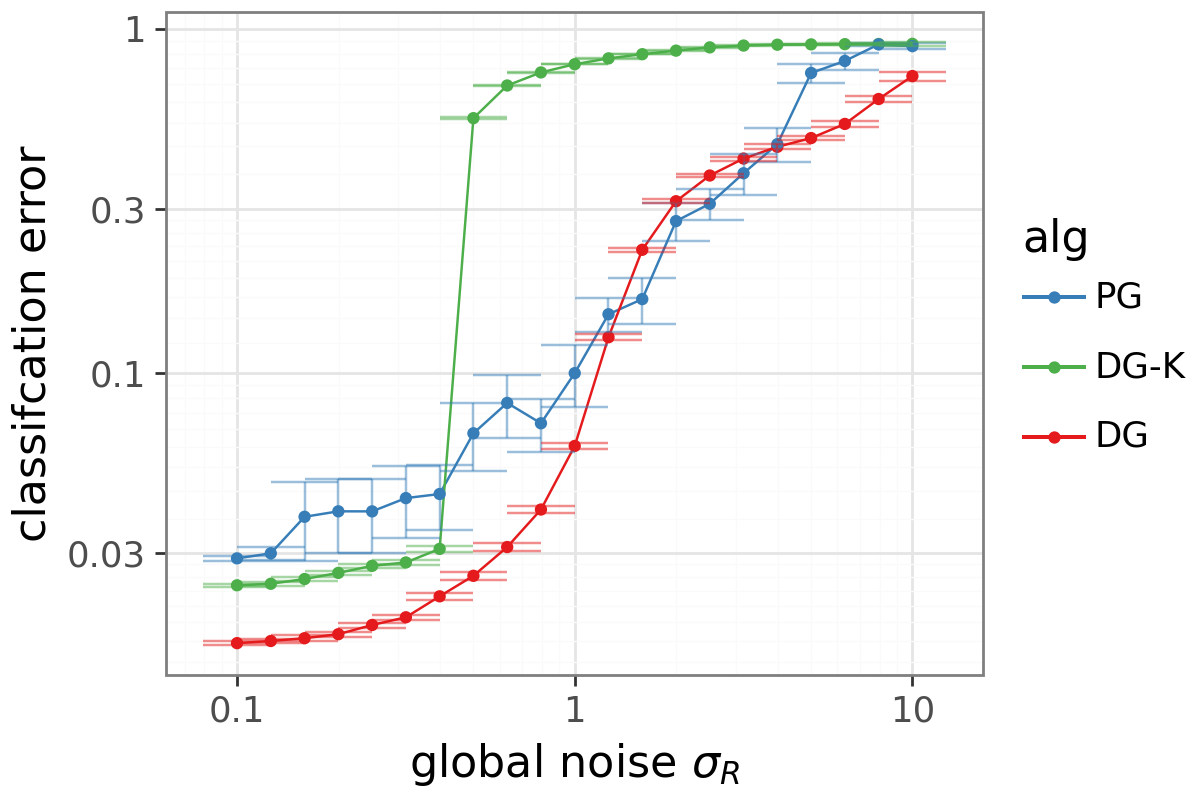}
    \vspace{-2mm}
    \caption{Homoskedastic $\sigma_R$: smooth joint degradation.}
    \label{fig:bandit_homo}
\end{subfigure}
\hfill
\begin{subfigure}[t]{0.48\columnwidth}
    \centering
    \includegraphics[width=\linewidth]{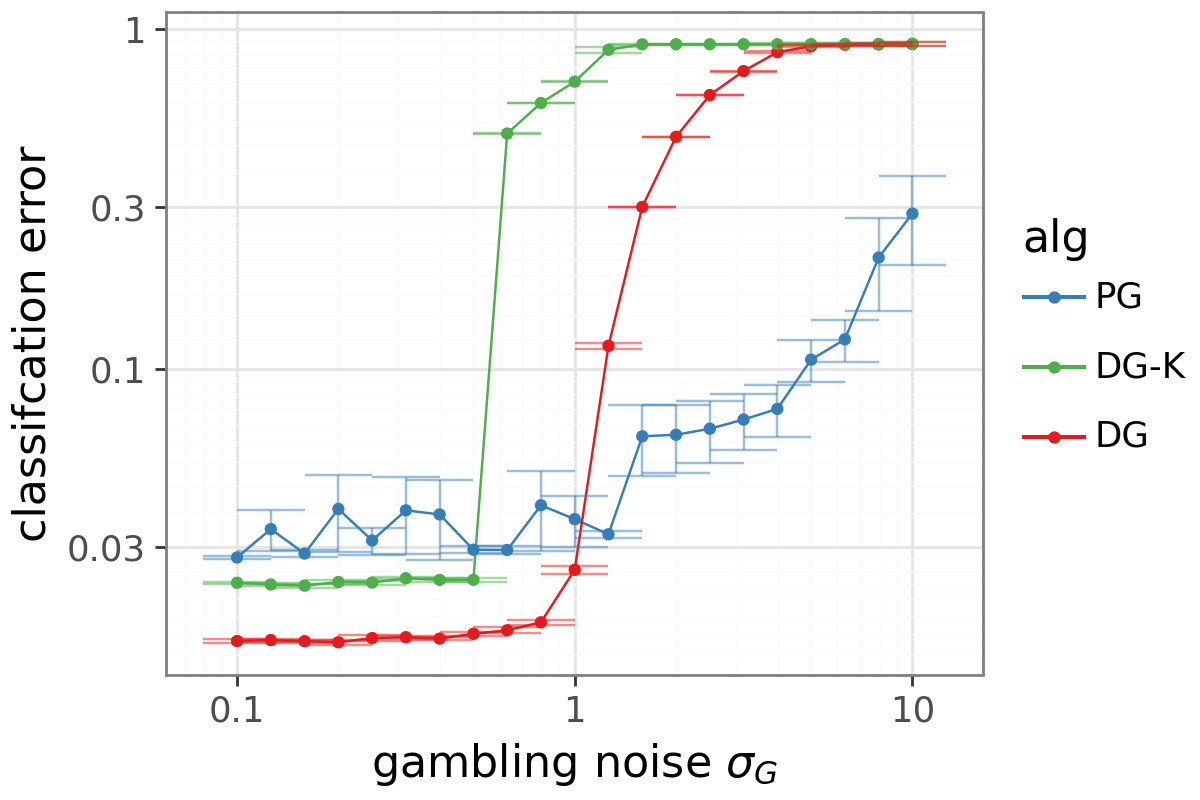}
    \vspace{-2mm}
    \caption{Gambling $\sigma_G$: sharp DG collapse.}
    \label{fig:bandit_hetero}
\end{subfigure}
\vspace{-1mm}
\caption{Gambling pathology on MNIST (Section~\ref{sec:mnist}).
(a)~Under homoskedastic noise, DG and PG degrade smoothly together.
(b)~Under differential noise on action $a{=}0$, DG collapses sharply near $\sigma_G \approx 1$ while PG degrades gracefully, consistent with the regimes of Proposition~\ref{prop:gambling}.}
\label{fig:bandit}
\vspace{-2mm}
\end{figure}

\vspace{-1mm}
\section{Token Reversal}
\label{sec:results}
\vspace{-1mm}

The bandit analysis showed that the Kondo gate can preserve useful gradient signal while removing much of the backward computation.
We now test whether the same mechanisms survive in sequence-model training, where compute efficiency matters most.
Token reversal~\citep{osband2026delightfulpolicygradient} asks a transformer to reverse a length-$H$ sequence drawn from a vocabulary of size $M$.
The task isolates autoregressive generation and token-level credit assignment, and lets us vary sequence length and vocabulary size independently.
As either $H$ or $M$ grows, the task becomes harder and informative events become rarer, so selective backward computation should become increasingly valuable.
The agent is a causal transformer with 3 layers, 4 heads, and embedding dimension 64 (about 100K parameters).
Each gradient step processes a batch of $100$ episodes ($10$ prompts ${\times}$ $10$ responses); full experimental details are in Appendix~\ref{app:token_reversal}.
Code is available at \codeurl.

\begin{figure}[ht!]
\centering
\vspace{-1mm}
\includegraphics[width=0.55\columnwidth]{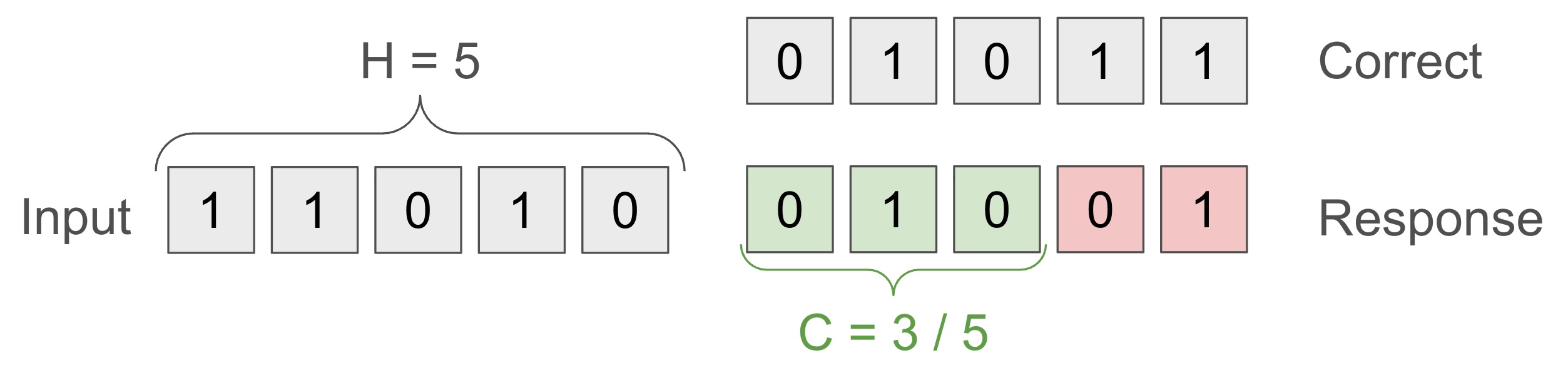}
\vspace{-2mm}
\caption{Token reversal ($M{=}2$, $H{=}5$): the agent must output the input in reverse; reward counts correct tokens up to the first error.}
\label{fig:binary_reverse_strict}
\vspace{-1mm}
\end{figure}

We compare PG, PPO, PMPO, DG (full delightful gradient), and two Kondo variants.
DG-K ($\rho = 3\%$) imposes a fixed backward-compute budget, keeping only the top $3\%$ of tokens by delight.
DG-K ($\lambda = 0$) gates on the sign of delight and adapts automatically as the policy improves.
We ask whether DG-K keeps DG's learning quality while collapsing the backward cost.
PPO coincides with PG in this setting: each batch is fresh and receives a single update, so the PPO ratio is one and clipping never activates, and all methods share the same grouped baseline with no learned critic.
PPO departs from PG under replay, multiple epochs, or stale actors~\citep{osband2026delightfuldistributedpolicygradient}; we keep those complications out to isolate compute selection.

Figure~\ref{fig:token_regret} shows the representative result.
In forward-pass space~(a), DG and both DG-K variants outperform PG, PPO, and PMPO by over an order of magnitude.
In backward-pass space~(b), the DG-K curves reach the same quality as DG with orders of magnitude fewer backward passes.

Figure~\ref{fig:scaling_vocab} sweeps vocabulary size $M$, exposing the tradeoff between the two gates.
In forward-pass space, the adaptive gate ($\lambda = 0$) tracks full DG and remains robust as $M$ grows, whereas the fixed gate ($\rho = 3\%$) becomes too aggressive at large vocabularies.
In backward-pass space, however, both Kondo variants still outperform all baselines.
The adaptive gate is the safer default when difficulty is unknown; the fixed gate is attractive when backward savings matter most.

Figure~\ref{fig:scaling_length} shows the main scaling result.
As sequence length grows, PG, PPO, and PMPO scale sublinearly, while DG and DG-K ($\rho = 3\%$) solve longer sequences at every forward budget; DG-K slightly exceeds full DG at the longest lengths.
In backward-pass space, DG-K ($\rho = 3\%$) solves $H^* \approx 29$ using a small fraction of the backward compute that DG needs to reach $H^* \approx 27$.
The gap widens with sequence length as correct sequences become exponentially rarer, and here the fixed-budget gate gives both the best learning quality and the lowest backward cost.

\begin{figure}[ht!]
\centering
\begin{subfigure}[t]{0.48\columnwidth}
    \centering
    \includegraphics[width=\linewidth]{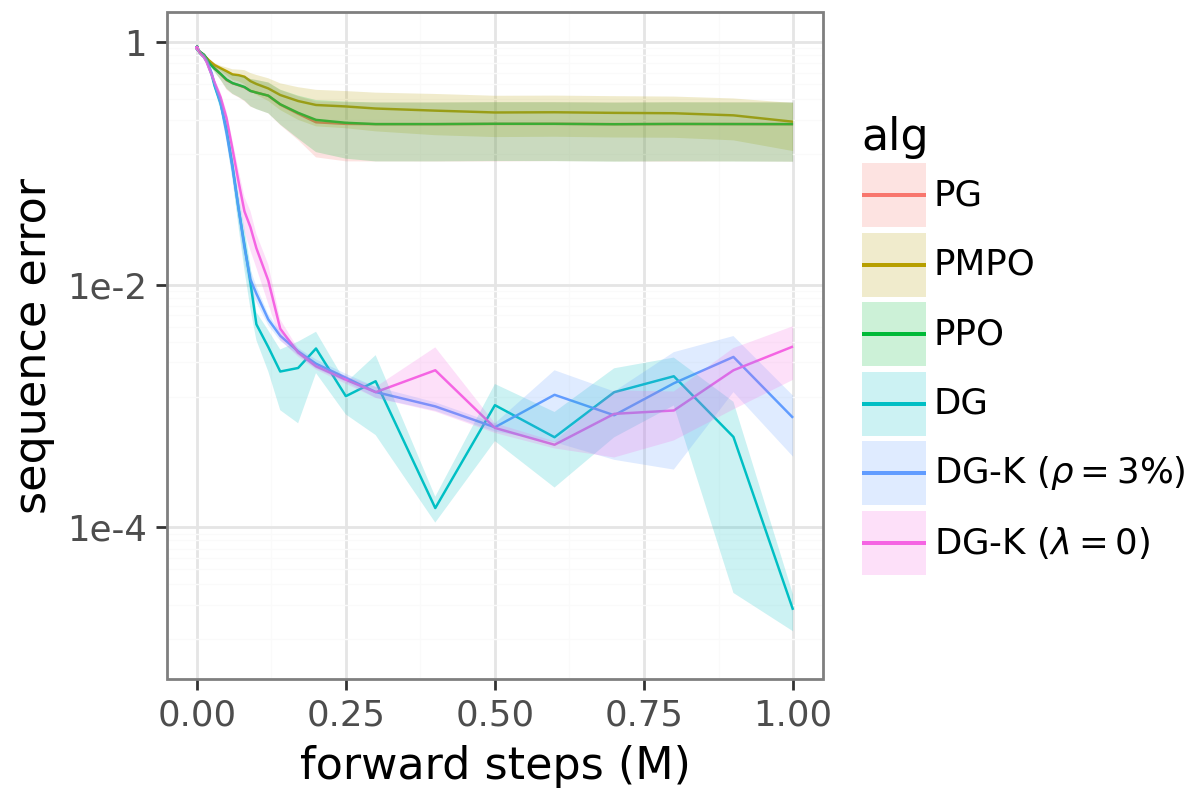}
    \vspace{-2mm}
    \caption{Forward passes: DG-K $\approx$ DG $\gg$ PG.}
    \label{fig:regret_forward}
\end{subfigure}
\hfill
\begin{subfigure}[t]{0.48\columnwidth}
    \centering
    \includegraphics[width=\linewidth]{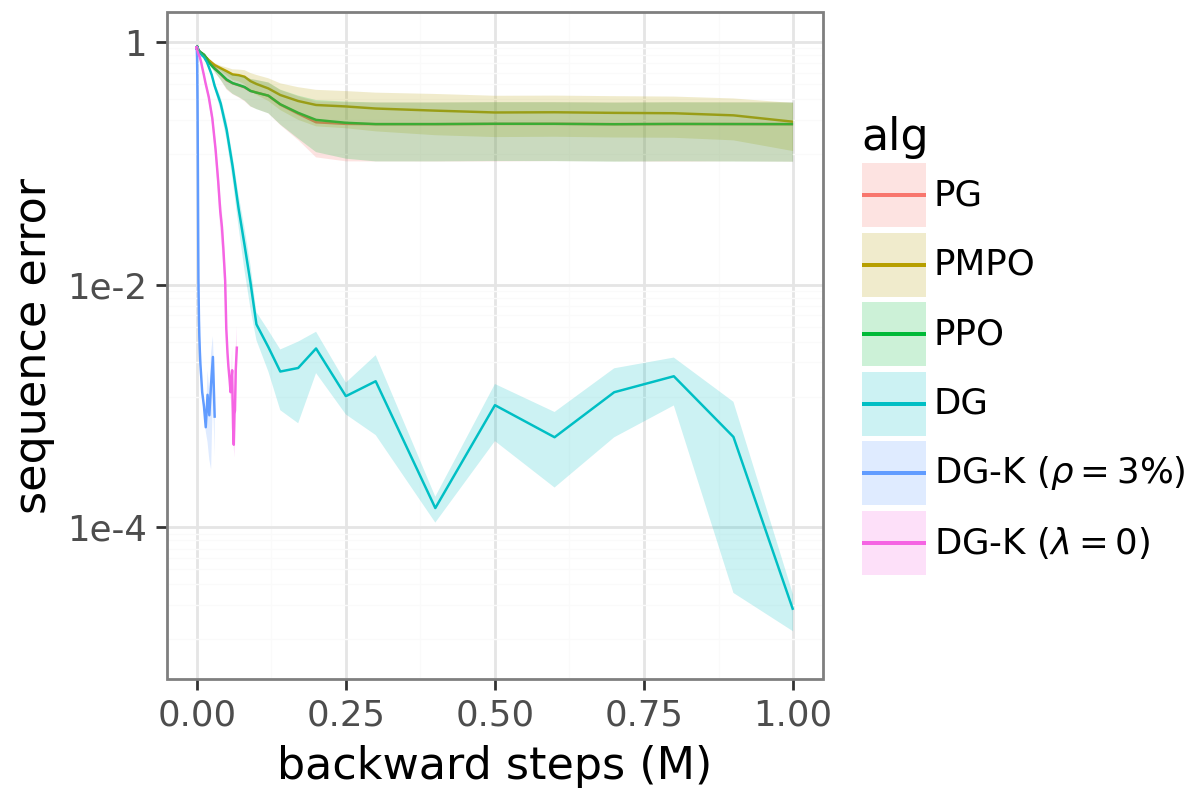}
    \vspace{-2mm}
    \caption{Backward passes: DG-K $\gg$ alternatives.}
    \label{fig:regret_back}
\end{subfigure}
\vspace{-2mm}
\caption{Token reversal learning curves ($H=10, M=2$).
10 seeds; shading $\pm 1$ s.e.}
\label{fig:token_regret}
\end{figure}

\begin{figure}[ht!]
\centering
\begin{subfigure}[t]{0.48\columnwidth}
    \centering
    \includegraphics[width=\linewidth]{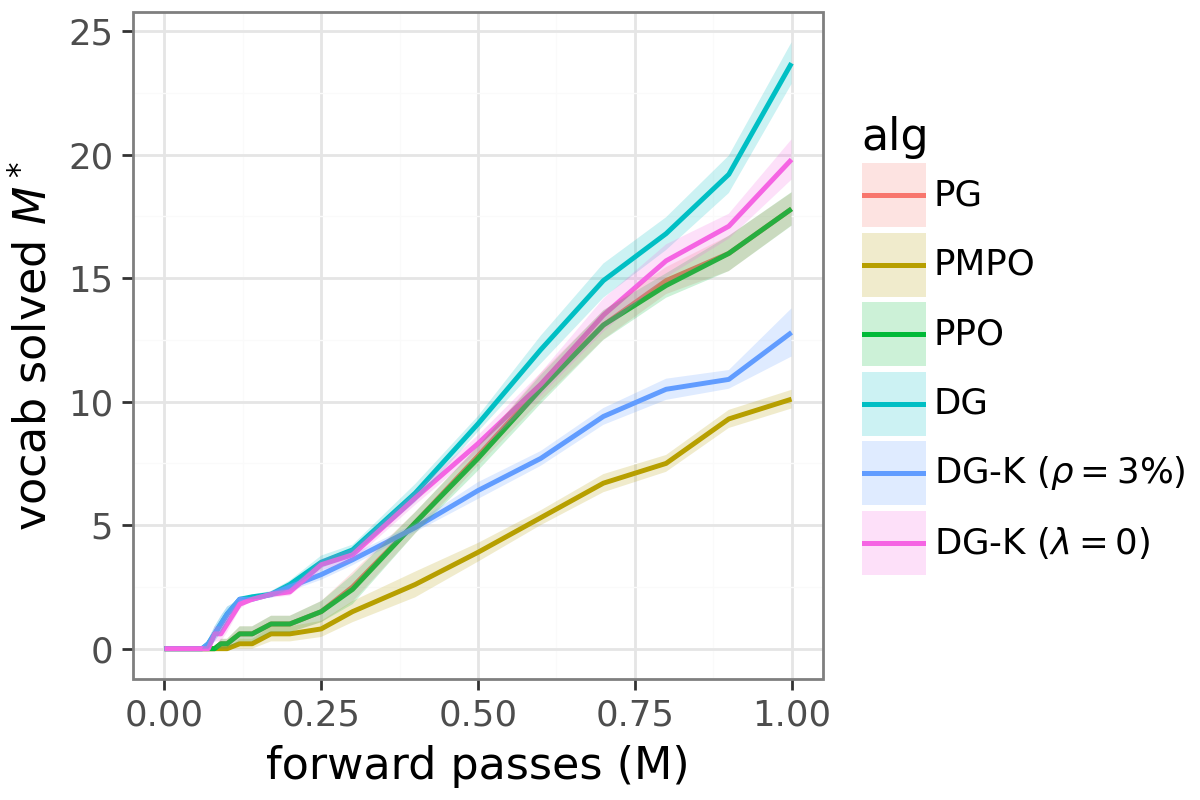}
    \vspace{-2mm}
    \caption{Vocab solved $M^*$ (forward passes).}
    \label{fig:vocab_scale_forward}
\end{subfigure}
\hfill
\begin{subfigure}[t]{0.48\columnwidth}
    \centering
    \includegraphics[width=\linewidth]{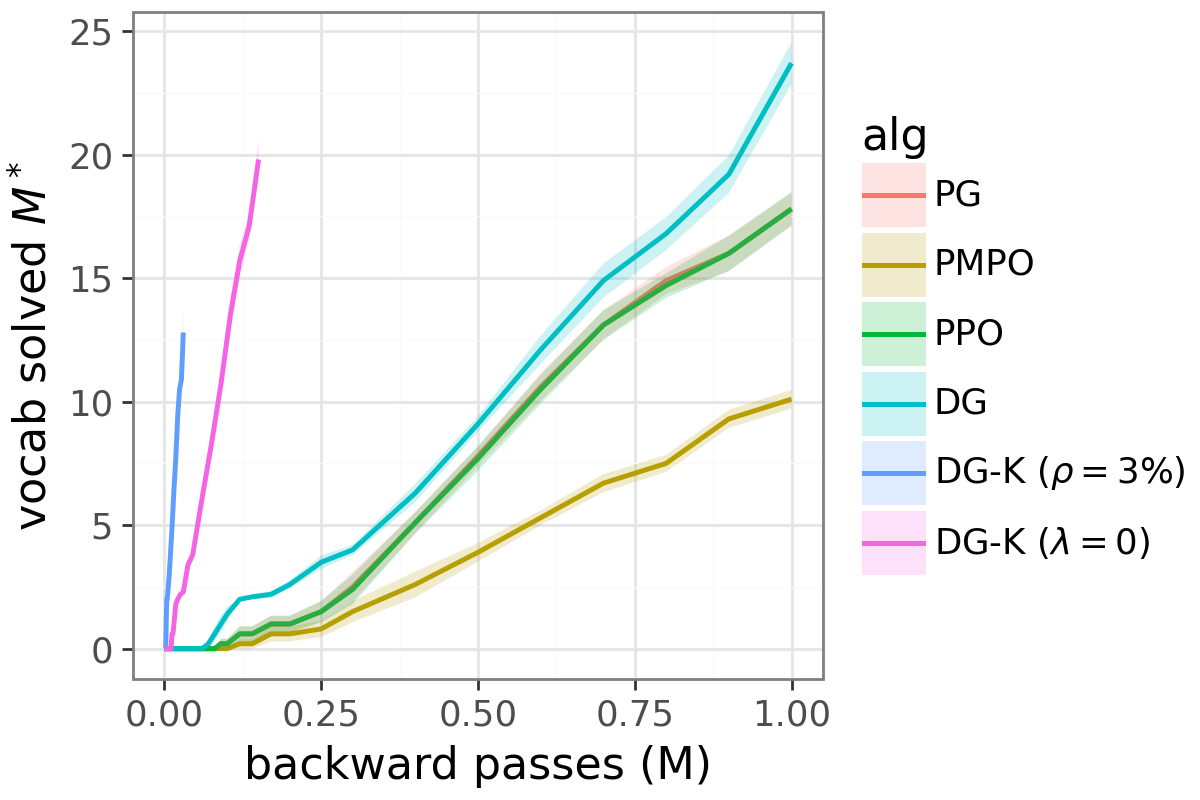}
    \vspace{-2mm}
    \caption{Vocab solved $M^*$ (backward passes).}
    \label{fig:vocab_scale_backward}
\end{subfigure}
\vspace{-2mm}
\caption{Scaling with vocabulary size: $M^*$ = largest vocabulary solved vs.\ compute.
Fixed $\rho = 3\%$ is too aggressive at large $M$; adaptive $\lambda = 0$ is robust and preserves backward-compute gains.}
\label{fig:scaling_vocab}
\end{figure}

\begin{figure}[ht!]
\centering
\begin{subfigure}[t]{0.48\columnwidth}
    \centering
    \includegraphics[width=\linewidth]{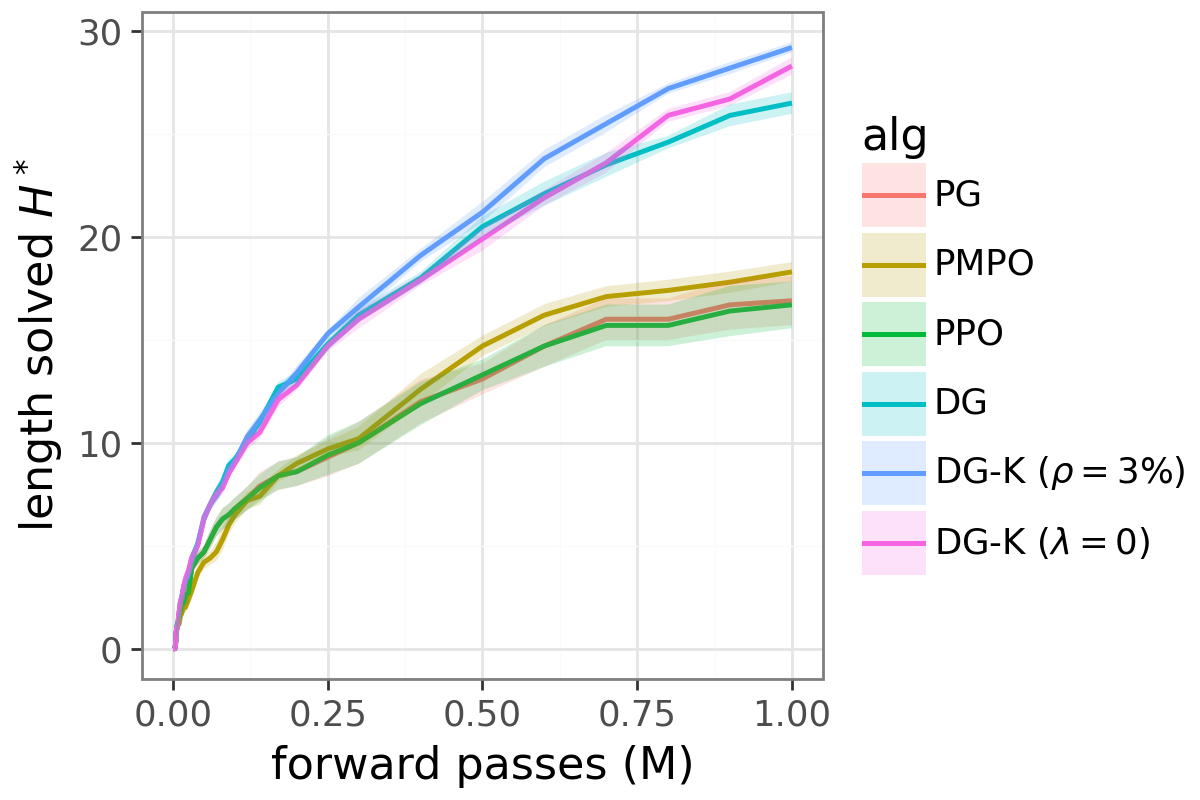}
    \vspace{-2mm}
    \caption{Length solved $H^*$ (forward passes).}
    \label{fig:length_scale_forward}
\end{subfigure}
\hfill
\begin{subfigure}[t]{0.48\columnwidth}
    \centering
    \includegraphics[width=\linewidth]{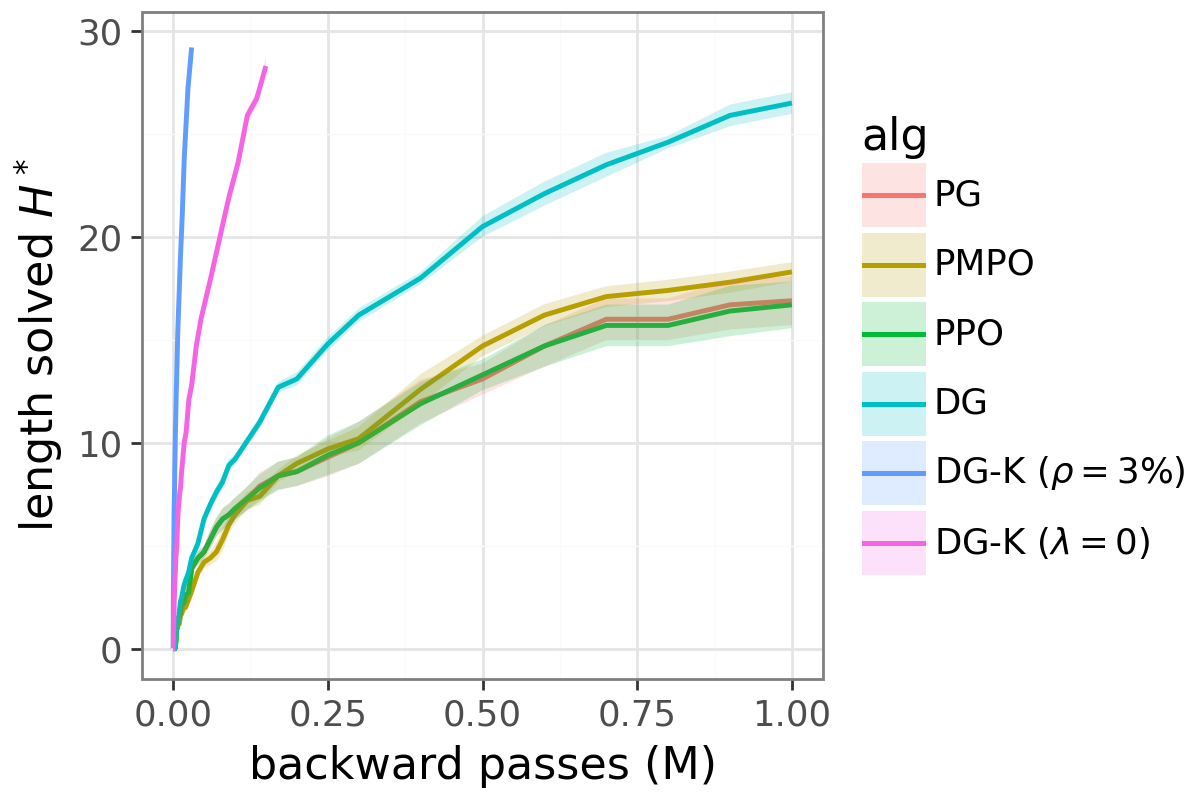}
    \vspace{-2mm}
    \caption{Length solved $H^*$ (backward passes).}
    \label{fig:length_scale_backward}
\end{subfigure}
\vspace{-2mm}
\caption{Scaling with sequence length: $H^*$ = longest sequence solved (reward $> 0.75$) vs.\ compute.
DG-K matches or exceeds DG and solves longer sequences at far lower backward cost.}
\label{fig:scaling_length}
\end{figure}

Across both scaling axes, the Kondo gate preserves DG's gains in forward-pass space and turns them into large gains in backward-pass space.
As sequence length and vocabulary size grow, the gate solves harder problems at equal backward compute.
This matches the bandit prediction that screening becomes more valuable when useful learning events are rare.
The adaptive gate ($\lambda = 0$) is reliable across problem difficulties; the fixed gate ($\rho = 3\%$) delivers the largest savings when tuned per task.

\section{Related Work}
\label{sec:related}

We situate the Kondo gate among existing approaches to compute-efficient training, priority sampling, and robust policy gradients.
The Kondo gate extends selective backpropagation to policy gradients, using advantage and surprisal instead of loss to decide which terms to compute.

\paragraph{Selective backpropagation and curriculum learning.}
Selective-Backprop~\citep{jiang2019accelerating} uses forward-pass losses to skip backward passes on easy examples, and importance sampling~\citep{katharopoulos2018not} favors high-loss samples, but loss is agnostic to gradient direction: high loss need not imply high learning value.
Curriculum learning~\citep{bengio2009curriculum,graves2017automated} similarly prioritizes training examples by difficulty, but selects which data to present rather than which gradients to compute.
Delight instead combines value and surprise, and Proposition~\ref{prop:delight_dominance} shows that its sign remains aligned with usefulness across regimes in which additive alternatives can mis-rank samples.

\paragraph{Prioritized experience replay.}
PER~\citep{schaul2016per} and its distributed extension Ape-X~\citep{horgan2018distributed} prioritize replay transitions by TD error, which conflates epistemic and aleatoric uncertainty.
The Kondo gate prioritizes \emph{gradient computation} instead of replay, and its priority signal comes from a single forward pass.

\paragraph{Speculative decoding.}
Speculative decoding~\citep{leviathan2023fast} skips expensive inference steps using a cheaper draft model.
The Kondo gate is the training counterpart, skipping expensive backward passes based on a cheap forward-pass signal.
Section~\ref{sec:compute_efficiency} shows that delight tolerates moderate noise, which supports the analogy.

\paragraph{UCB, active learning, and optimistic exploration.}
UCB-style methods~\citep{auer2002finite} encourage exploration through an additive bonus.
Delight's surprisal term is instead multiplicative: it prioritizes events that are both valuable and unexpected.
Proposition~\ref{prop:delight_dominance} formalizes the difference: additive mixtures require regime-dependent tuning, while the product remains sign-consistent.
The gate can also be viewed as within-batch active learning, but it requires no acquisition model: the selection criterion is a by-product of the forward pass.

\paragraph{GRPO, AWR, PMPO.}
These methods weight by advantage or exponentiated advantage but are surprisal-blind~\citep{shao2024deepseekmath,peng2019advantage,abdolmaleki2024preference,schulman2017proximal,schulman2015trust}.
As $M^H$ grows, gradient budget concentrates on predictable tokens whose advantage happened to be nonzero, starving rare informative events.
The Kondo gate inherits DG's surprisal sensitivity and uses it to decide which backward passes are worth computing.
The same principle applies to RLHF pipelines~\citep{ouyang2022training,ziegler2019fine}, where backward passes through large models are especially expensive.

\section{Conclusion}
\label{sec:conclusion}

In the settings studied here, the forward pass is enough to decide whether a backward pass is worth computing.
Delight, the product of advantage and surprisal, provides that signal, and the price $\lambda$ sets the quality--cost trade-off.
Across MNIST, bandits, and transformer token reversal, the Kondo gate skips most backward passes while retaining nearly all of DG's learning quality, yielding large gains over PG, PPO, and PMPO.

In our sequence tasks most gradient computations are redundant, and screening before backpropagation becomes more valuable as useful learning events become rarer and backward passes more expensive.
The gate's robustness to approximate delight motivates a speculative-training direction: a cheap screening pass selects samples for a more expensive differentiable pass.

Because we count selected gradient terms instead of measuring wall-clock time, realizing the savings requires a two-stage implementation that scores samples without gradients and backpropagates only the selected ones, since a naive loss mask runs the same dense kernels; for token-level selection, a kept loss still backpropagates through its prefix, so selected terms do not translate directly into transformer FLOPs.
The gated estimator is deliberately biased, and general convergence remains open, including under stochastic rewards, inaccurate advantages, and batch-dependent thresholds.
Our tasks have unambiguous rewards and lack the ambiguity and stochasticity of preference-based learning, and the gambling regime of Proposition~\ref{prop:gambling} marks where delight itself misleads.
Exact delight needs the realized action and reward, so it can allocate gradient computation but cannot guide sampling; predicting delight before sampling defines a separate exploration method~\citep{osband2026delightfulexploration}.
We also have not measured how gating affects the KL--performance trade-off or plasticity.
The central open question is scale: do the same gains survive in modern large-model training?

\section*{Acknowledgements and Disclosures}
We thank Ben Van Roy and Satinder Singh for detailed feedback on earlier drafts.
We are grateful to Raia Hadsell, Zoubin Ghahramani, Demis Hassabis, and Satinder Singh for building the research environment at Google DeepMind that made this work possible.
This work was funded by Google DeepMind, where the author is employed.
The author has no other competing interests to declare.

\bibliography{references}
\bibliographystyle{plainnat}


\newpage
\appendix

\section{MNIST Diagnostic}
\label{app:mnist}

We provide experimental details and supplementary figures for the MNIST experiments in Section~\ref{sec:mnist}.

\subsection{Experimental Details}
\label{app:mnist_details}

\paragraph{Architecture.}
The policy is a two-layer MLP with 100 hidden units per layer and softmax output over 10 actions (digits 0--9).
The environment is an MNIST contextual bandit: the agent observes an image, selects an action, and receives reward $r = \mathbb{I}\{a = y\}$ where $y$ is the true label.

\paragraph{Optimization.}
All methods use Adam with a learning rate swept over $\{10^{-4}, 3 \times 10^{-4}, 10^{-3}, 3 \times 10^{-3}, 10^{-2}\}$.
Each gradient step uses a batch of $B = 100$ samples drawn with replacement from the training set.
Training runs for $10{,}000$ gradient steps with validation every 100 steps on the full 10{,}000-image test set.
All main-body figures average over 30 seeds; shading shows $\pm 1$ standard error.

\paragraph{Baseline and advantage.}
All methods use the expected-confidence baseline $b = \sum_a \pi(a)^2$, the reward the policy would expect if its own sample were the label; it needs no access to the true label.
Since $b \in (0,1)$, this gives advantage $U(y^*) = 1 - b > 0$ for the correct action and $U(a \neq y^*) = -b < 0$ for incorrect actions.

\paragraph{Kondo gate.}
The gate rate $\rho$ is implemented by keeping the top $\rho B$ samples by delight within each batch.
We sweep $\rho \in \{0.01, 0.02, 0.03, 0.05, 0.1, 0.2, 0.3, 0.5, 1.0\}$; $\rho = 1$ recovers standard PG.
For the priority comparison (Figure~\ref{fig:mnist_priority}), we additionally test advantage-only, surprisal-only, absolute-advantage, uniform (random subsampling), and additive $(1{-}\alpha)U + \alpha\ell$ with $\alpha \in \{0, 0.1, \ldots, 1\}$.

\paragraph{Gambling experiment.}
For the gambling pathology (Figure~\ref{fig:bandit}), action $a = 0$ is designated the gamble: whenever the agent predicts $0$ (regardless of true label), its reward receives additive $\mathcal{N}(0, \sigma_G^2)$ noise.
Homoskedastic noise $\sigma_R$ is applied to all actions.
We sweep $\sigma_G \in \{0, 0.5, 1.0, 1.5, 2.0\}$ and $\sigma_R \in \{0, 0.5, 1.0, 2.0, 5.0\}$.

\subsection{Learning Rate Sensitivity and Test Error}
\label{app:mnist_lr}

Figure~\ref{fig:mnist_lr} sweeps the learning rate for PG, DG, and DG-K ($\rho = 0.03$).
All three methods share the same optimum at $\mathrm{lr} = 10^{-3}$, with DG dominating across the entire range and DG-K close behind.
The pattern is nearly identical for training error~(a) and test error~(b): the Kondo gate is not exploiting a train/test gap.

\begin{figure}[ht!]
\centering
\begin{subfigure}[t]{0.48\columnwidth}
    \centering
    \includegraphics[width=\linewidth]{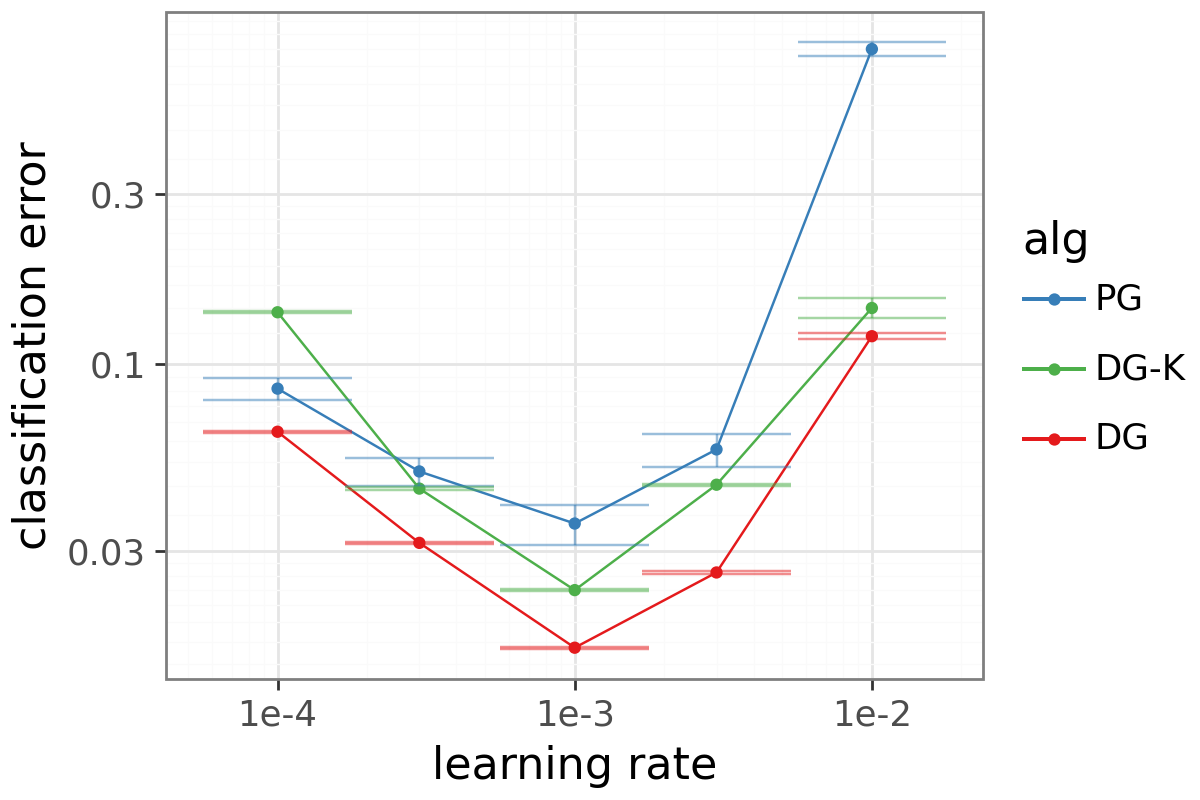}
    \caption{Training error vs.\ learning rate.}
    \label{fig:mnist_lr_train}
\end{subfigure}
\hfill
\begin{subfigure}[t]{0.48\columnwidth}
    \centering
    \includegraphics[width=\linewidth]{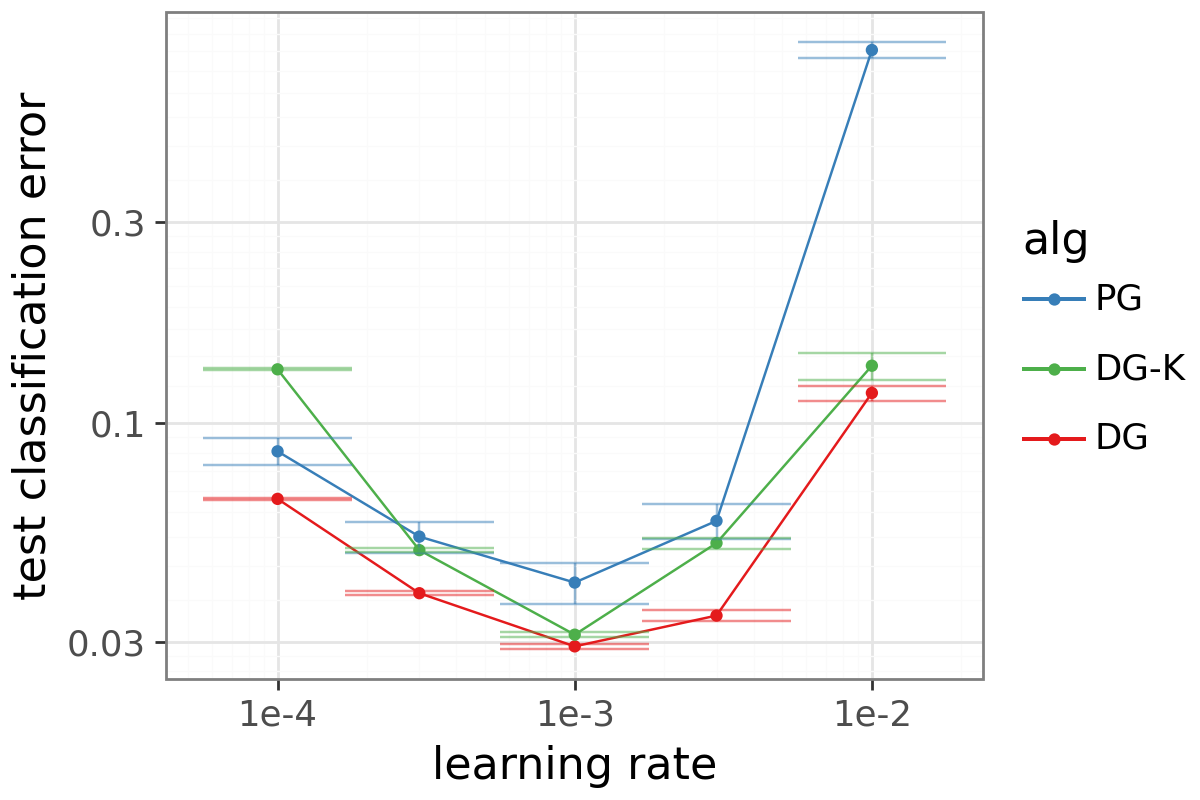}
    \caption{Test error vs.\ learning rate.}
    \label{fig:mnist_lr_test}
\end{subfigure}
\caption{Learning rate sweep on MNIST.
All methods are optimal near $\mathrm{lr} = 10^{-3}$; DG dominates across the range.
Training and test error track closely, confirming no train/test gap.}
\label{fig:mnist_lr}
\end{figure}

Figure~\ref{fig:mnist_test} replicates the main-body comparison (Figure~\ref{fig:mnist}) using test classification error rather than training error.
The same pattern holds: DG-K matches DG in forward-pass space and dominates by two orders of magnitude in backward-pass space.
This confirms that the Kondo gate's advantage is not specific to the RL reward signal; it transfers to the supervised-learning notion of generalization error that is more standard in MNIST benchmarks.

\begin{figure}[ht!]
\centering
\begin{subfigure}[t]{0.48\columnwidth}
    \centering
    \includegraphics[width=\linewidth]{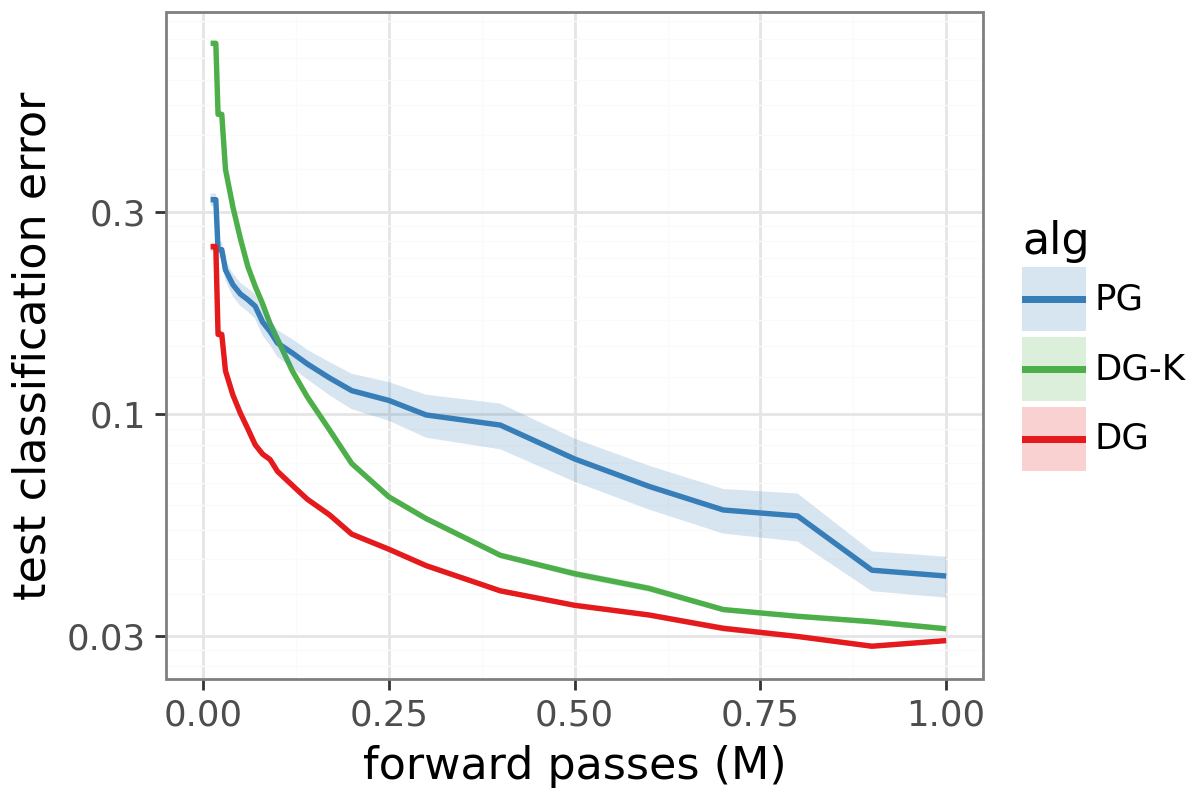}
    \caption{Forward passes: test error.}
    \label{fig:mnist_forward_test}
\end{subfigure}
\hfill
\begin{subfigure}[t]{0.48\columnwidth}
    \centering
    \includegraphics[width=\linewidth]{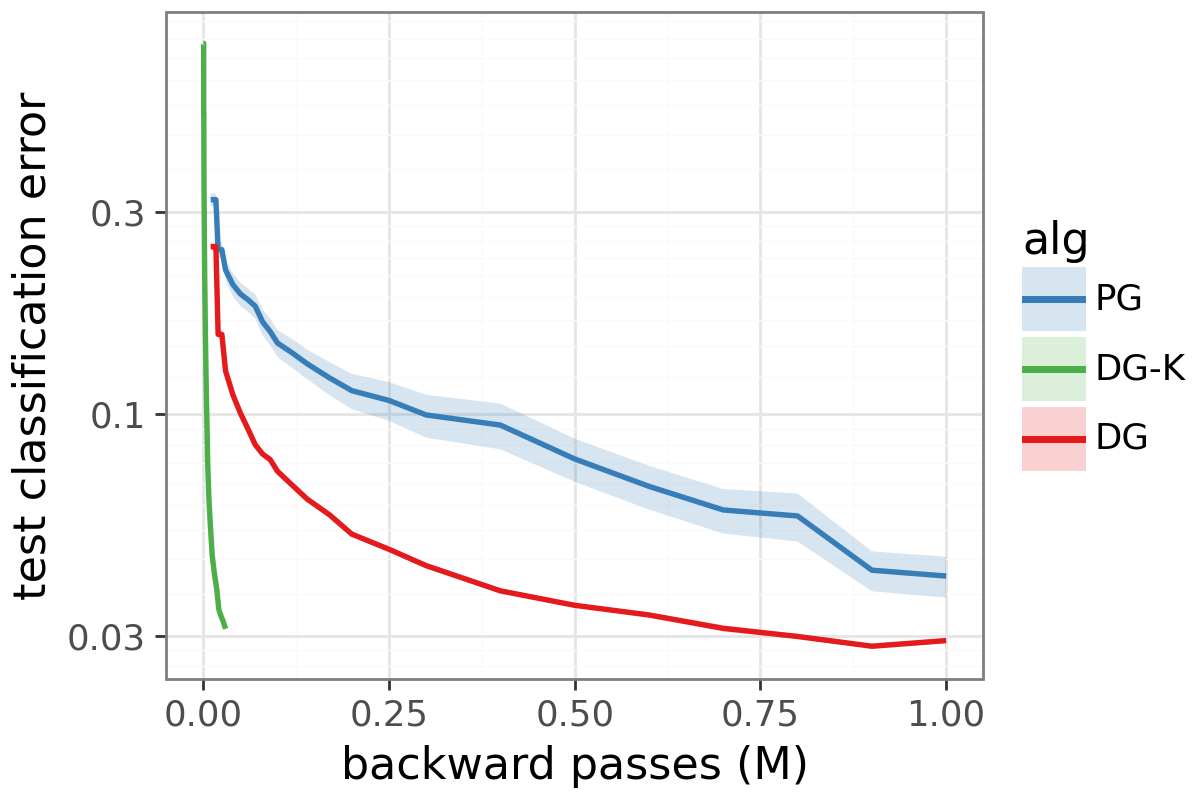}
    \caption{Backward passes: test error.}
    \label{fig:mnist_backward_test}
\end{subfigure}
\caption{Test classification error on MNIST at $\rho = 0.03$.
Same comparison as Figure~\ref{fig:mnist} but with held-out test error.
DG-K matches DG in forward-pass space and dominates in backward-pass space; the gate's advantage generalizes beyond training error.}
\label{fig:mnist_test}
\end{figure}

\subsection{Baseline Robustness}
\label{app:mnist_baselines}

The main-body results use the expected-confidence baseline $b = \sum_a \pi(a)^2$, which \citet{osband2026delightfulpolicygradient} identifies as the natural choice for MNIST.
To check that the Kondo gate's advantage is not an artifact of this choice, Figures~\ref{fig:mnist_baselines_forward} and~\ref{fig:mnist_baselines_backward} repeat the comparison under four baselines: zero ($b = 0$), constant ($b = 0.5$), expected confidence ($b = \sum_a \pi(a)^2$), and oracle ($b = \mathbb{E}[R \mid x]$ using the true label).

The general pattern is the same across all four baselines: DG dominates PG, DG-K matches DG in forward-pass space, and DG-K dominates in backward-pass space.
Under the zero baseline, DG-K actually outperforms DG on forward passes; we report the expected baseline in the main body to avoid cherry-picking.

\begin{figure}[ht!]
\centering
\includegraphics[width=\columnwidth]{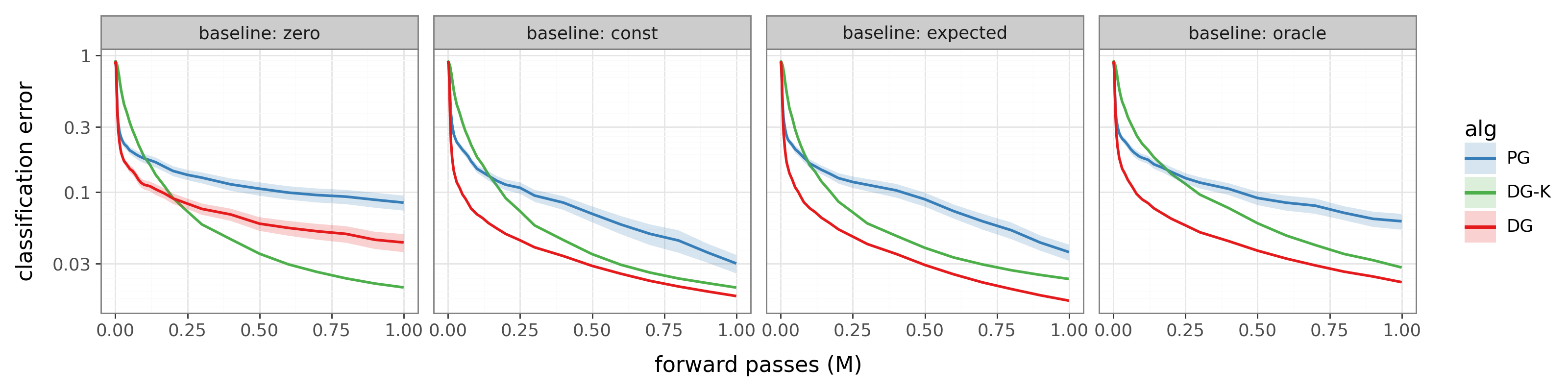}
\caption{Forward-pass comparison across baselines on MNIST at $\rho = 0.03$.
The Kondo gate matches or exceeds DG under all four baselines.
Under the zero baseline, DG-K even outperforms DG in forward-pass space.}
\label{fig:mnist_baselines_forward}
\end{figure}

\begin{figure}[ht!]
\centering
\includegraphics[width=\columnwidth]{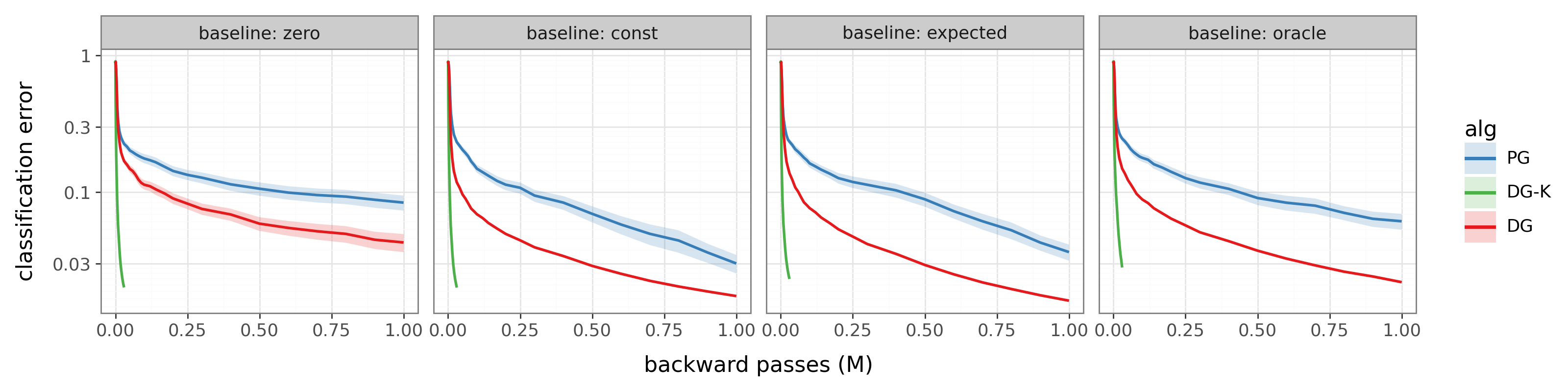}
\caption{Backward-pass comparison across baselines on MNIST at $\rho = 0.03$.
The Kondo gate dominates in backward-pass space under all four baselines: the two-orders-of-magnitude advantage is not baseline-dependent.}
\label{fig:mnist_baselines_backward}
\end{figure}

\subsection{Gate Selection Profile}
\label{app:mnist_histogram}

Figure~\ref{fig:mnist_histogram} shows the empirical CDF of $\pi(y^*)$ for kept vs.\ skipped samples at three stages of training, aggregated over 100 batches (10{,}000 samples per stage).
At step 100, both distributions nearly overlap: the policy is too uncertain currently for delight to discriminate.
By step 1{,}000, clear separation emerges: kept samples (red) have systematically lower $\pi(y^*)$ than skipped samples (blue), demonstrating first-order stochastic dominance.
The gate targets the \emph{learning frontier}: samples the policy predicts correctly but without confidence.
By step 10{,}000, both distributions concentrate near $\pi(y^*) \approx 1$, but the kept distribution remains shifted left, finding the few remaining hard cases.

\begin{figure}[ht!]
\centering
\includegraphics[width=\columnwidth]{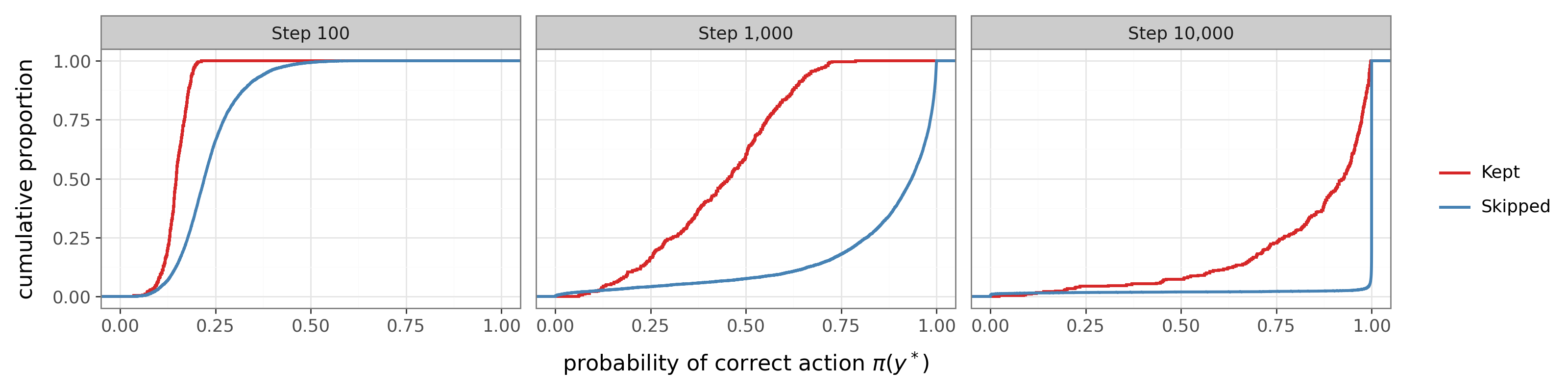}
\caption{Empirical CDF of $\pi(y^*)$ for kept vs.\ skipped samples at three training stages ($\rho = 0.03$, 10{,}000 samples per stage).
Kept samples have systematically lower $\pi(y^*)$: the gate selects the learning frontier where the model is correct but uncertain.}
\label{fig:mnist_histogram}
\end{figure}

\subsection{Kept vs.\ Skipped Images}
\label{app:mnist_kept}

Figure~\ref{fig:mnist_kept} shows images the Kondo gate keeps versus those it skips at $\rho = 0.03$ across three training stages.
Each image is annotated with its true label $y$, the action selected $a$, and the probability of the correct action $p = \pi(y^*)$.

At step 100~(a), the model is too uncertain for the gate to make meaningful distinctions: both rows contain diverse digits with low $p$.
By step 1{,}000~(b), separation emerges: kept images have moderate $p$ (the model is learning but not yet confident), while skipped images are either already solved ($p \approx 1$) or hopelessly misclassified.
At step 10{,}000~(c), the contrast is sharpest: kept images are correctly classified ($a = y$) but with $p < 1$, the learning frontier.
Skipped images include both solved cases ($p = 1.00$) and a few confidently wrong predictions ($a \neq y$, very low $p$), where the negative advantage makes delight negative.

\begin{figure}[ht!]
\centering
\begin{subfigure}[t]{\columnwidth}
    \centering
    \includegraphics[width=\linewidth]{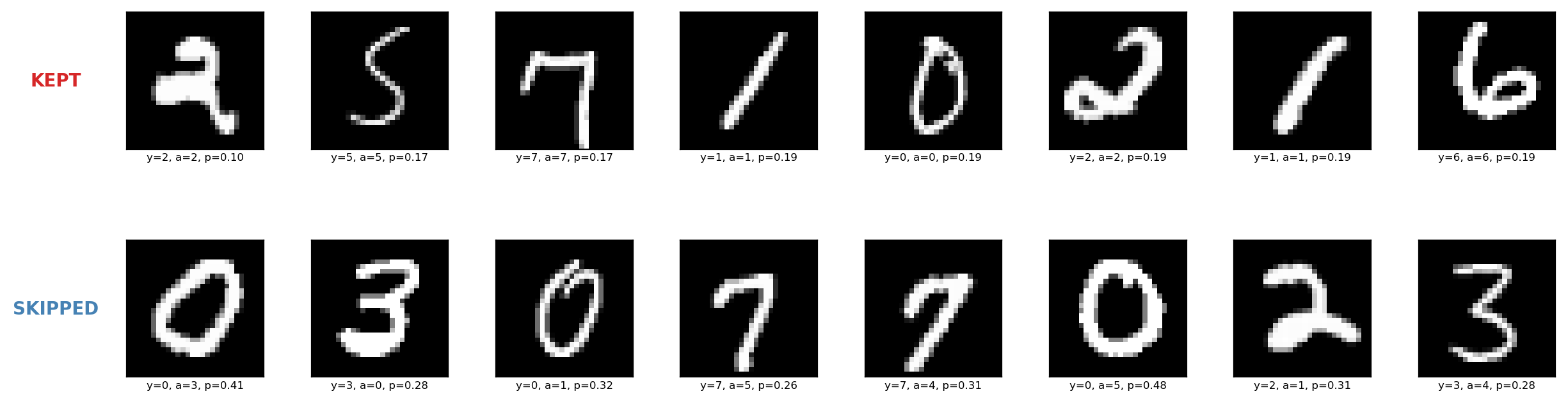}
    \caption{Step 100: the model is too uncertain to meaningfully discriminate.}
    \label{fig:mnist_kept_early}
\end{subfigure}
\\[0.5em]
\begin{subfigure}[t]{\columnwidth}
    \centering
    \includegraphics[width=\linewidth]{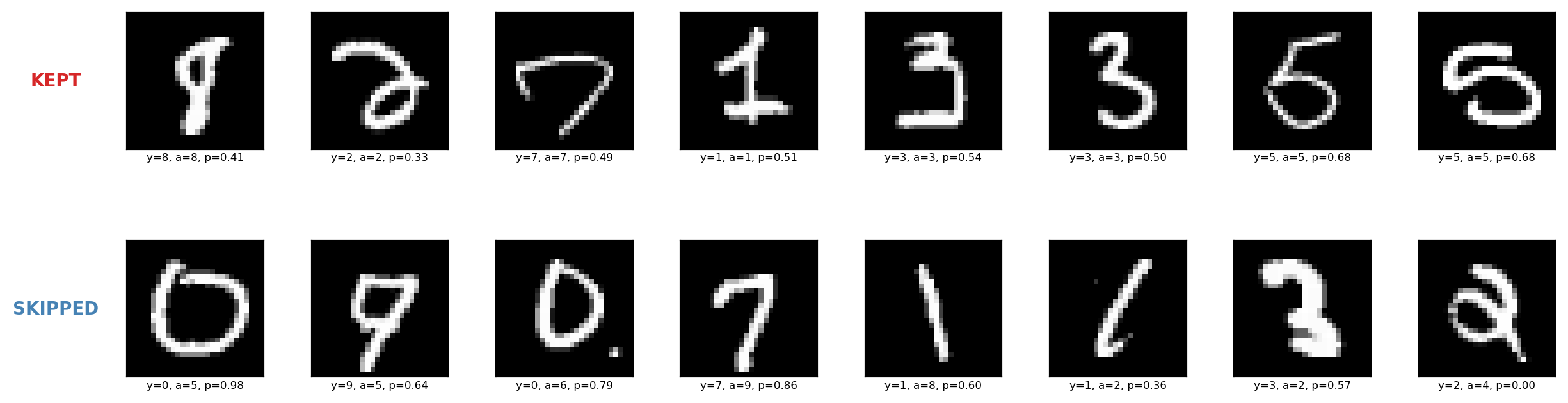}
    \caption{Step 1{,}000: separation emerges; kept images have moderate $\pi(y^*)$.}
    \label{fig:mnist_kept_mid}
\end{subfigure}
\\[0.5em]
\begin{subfigure}[t]{\columnwidth}
    \centering
    \includegraphics[width=\linewidth]{figures/kept_discarded_step10000.png}
    \caption{Step 10{,}000: kept images are correct but uncertain ($a = y$, $p < 1$); skipped are solved or confidently wrong.}
    \label{fig:mnist_kept_late}
\end{subfigure}
\caption{Images kept vs.\ skipped by the Kondo gate at $\rho = 0.03$ across three training stages.
Each image shows ground truth $y$, selected action $a$, and $p = \pi(y^*)$.
The gate progressively concentrates compute on the learning frontier.}
\label{fig:mnist_kept}
\end{figure}

\subsection{Delight Noise Robustness}
\label{app:delight_noise_absolute}

Figure~\ref{fig:robust_delight_absolute} complements Figure~\ref{fig:robust_delight} (Section~\ref{sec:compute_efficiency}) by showing delight noise robustness on an absolute rather than relative scale.
The qualitative pattern is the same: DG tolerates substantial noise before degrading; DG-K is more fragile.
The relative-scale figure in the main text is more interpretable for the approximate-delight argument, since the noise level is normalized by the signal magnitude.

\begin{figure}[ht!]
\centering
\includegraphics[width=0.6\columnwidth]{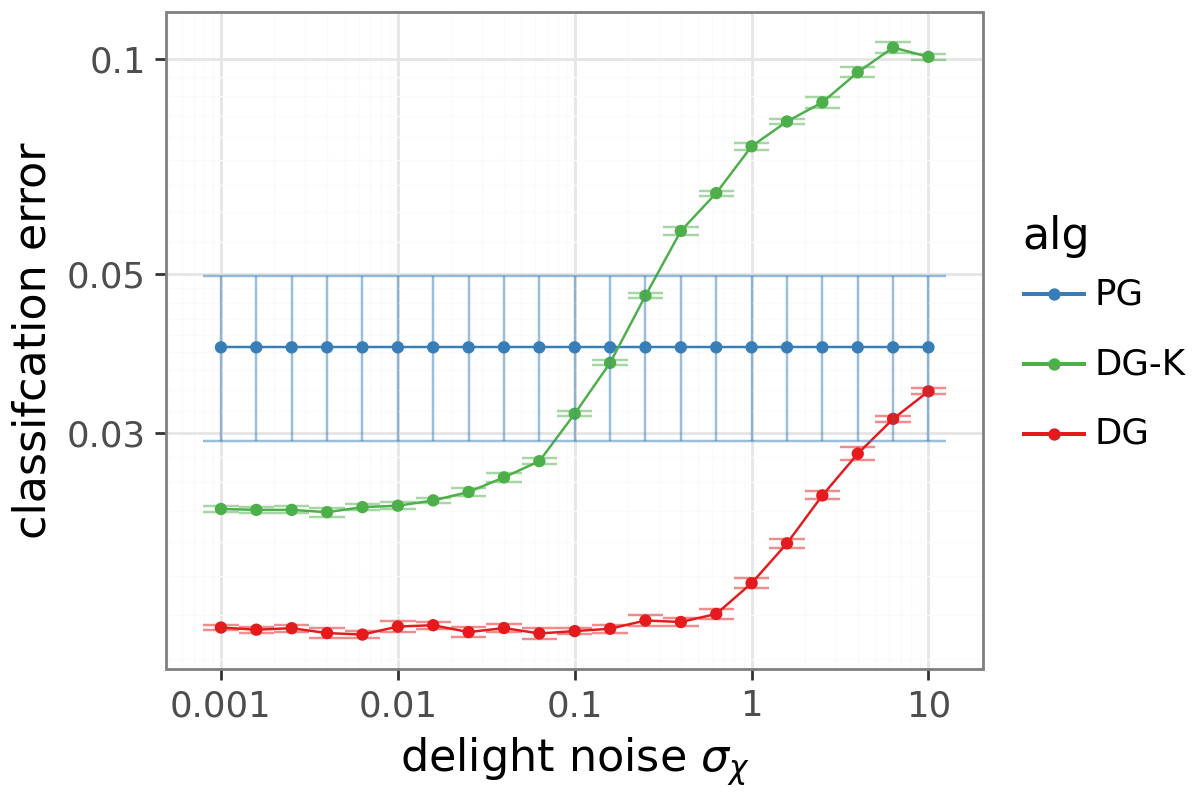}
\caption{Delight noise robustness (absolute scale): classification error vs.\ absolute noise $\sigma_\chi$.
Complements Figure~\ref{fig:robust_delight}; same qualitative pattern on an absolute scale.}
\label{fig:robust_delight_absolute}
\end{figure}

\section{Derivation of the Gate Weight}
\label{app:gate_derivation}

We derive the closed-form gate weight used in Algorithm~\ref{alg:kondo} (Section~\ref{sec:kondo_implementation}).

The meta-learner's objective for a single sample is
\[
f(w) = \delight \, w - \lambda \, w + \mix \, H(w), \qquad w \in [0,1],
\]
where $H(w) = -w \log w - (1{-}w) \log(1{-}w)$ is binary entropy.
Differentiating and setting to zero:
\[
f'(w) = (\delight - \lambda) + \mix \big(\log(1{-}w) - \log w\big) = 0
\quad \Longrightarrow \quad
\frac{\delight - \lambda}{\mix} = \log \frac{w}{1-w}.
\]
The right-hand side is the logit of $w$, so inverting gives $w^* = \sigma\!\big((\delight - \lambda)/\mix\big)$.
Since $H$ is strictly concave and the linear terms preserve concavity, $f$ is strictly concave, and the maximizer is unique.

The expected gated update for sample $t$ is therefore
\[
\mathbb{E}[G_t\, g_t]
= w^*_t \cdot U_t \nabla_\theta \log \pi_\theta(A_t \mid \mathcal{H}_t),
\qquad
w^*_t = \sigma\!\big((\delight_t - \lambda)/\mix\big).
\]
When $\lambda = 0$, this coincides with the DG update: the stochastic Kondo gate is a sparse Monte Carlo estimator of DG.

\section{Tabular Analysis}
\label{app:proofs}

This appendix provides the geometric setup and proofs for the three propositions in Section~\ref{sec:bandit}: the Kondo gate Pareto improvement (Proposition~\ref{prop:pareto_formal}), delight sign-consistency (Proposition~\ref{prop:delight_dominance}), and the gambling pathology (Proposition~\ref{prop:gambling}).
Propositions~\ref{prop:pareto_formal} and~\ref{prop:delight_dominance} share a common bandit setup and geometry lemma, stated first; Proposition~\ref{prop:gambling} states its own two-armed stochastic model.

\subsection{Setup and Geometry Lemma}

The following assumption underpins Propositions~\ref{prop:pareto_formal} and~\ref{prop:delight_dominance}.

\begin{assumption}[Symmetric $K$-armed bandit]
\label{assump:bandit}
$K$ arms ($K \ge 3$), a single correct arm $y^*$.
Deterministic reward $R = \mathbb{I}\{A = y^*\}$ (Section~\ref{sec:gambling_proof} extends to stochastic).
Softmax policy with success probability $p := \pi(y^*) \in (0,1)$ and uniform incorrect: $\pi(a) = (1{-}p)/(K{-}1)$ for $a \neq y^*$.
Baseline $b \in (0,1)$.
Score $\phi_\pi(a) := e_a - \pi$ (logit-space gradient of $\log \pi(a)$).
Batch size $B$; normalized step $z^+ = z + \alpha\, \bar g / \|\bar g\|$.
\end{assumption}

Each sample yields a per-sample gradient $g(a) = U(a)\,\phi_\pi(a)$, where the advantage takes two values:
\begin{equation}
\label{eq:advantages}
U(y^*) = 1 - b > 0, \qquad U(a \neq y^*) = -b < 0.
\end{equation}

Correct-action gradients carry pure signal; incorrect-action gradients carry mostly noise.

\begin{lemma}[Softmax gradient geometry]
\label{lem:geometry}
Under Assumption~\ref{assump:bandit}, the true gradient is $\nabla_z J = p\,\phi_\pi(y^*)$.
Write $\Pi_\perp$ for orthogonal projection away from $\nabla_z J$, and $\Var_\perp(g) := \E\|\Pi_\perp(g - \E g)\|^2$.
\begin{enumerate}
\item \textbf{Correct action:}
    $\phi_\pi(y^*)$ is a positive scalar multiple of $\nabla_z J$ with $\|\phi_\pi(y^*)\| = \Theta(1-p)$.
    $\Pi_\perp(\phi_\pi(y^*)) = 0$.
\item \textbf{Incorrect action:}
    $\|\phi_\pi(a)\| = \Theta(1)$ and the cosine with $\nabla_z J$ is $\Theta(p)$:
    \[
    \frac{|\langle \phi_\pi(a),\, \nabla_z J \rangle|}{\|\phi_\pi(a)\| \cdot \|\nabla_z J\|} = \Theta(p).
    \]
    Each incorrect-action gradient term has $\Theta(1)$ perpendicular noise and only an $O(p)$ fraction aligned with the true gradient.
\end{enumerate}
\end{lemma}

\begin{proof}
\medskip\noindent\textbf{Correct action.}
$\phi_\pi(y^*) = e_{y^*} - \pi$.
Writing $p_a := (1{-}p)/(K{-}1)$ for the uniform incorrect probability:
$\|\phi_\pi(y^*)\|^2 = (1-p)^2 + (K{-}1) \cdot p_a^2 = (1-p)^2 \cdot K/(K{-}1) = \Theta((1-p)^2)$.
Since $\nabla_z J = p\,\phi_\pi(y^*)$, $\phi_\pi(y^*)$ is a positive scalar multiple of $\nabla_z J$, so $\Pi_\perp(\phi_\pi(y^*)) = 0$.

\medskip\noindent\textbf{Incorrect action.}
For $a \neq y^*$: $\phi_\pi(a) = e_a - \pi$ has $\|\phi_\pi(a)\|^2 = 1 - 2p_a + \|\pi\|^2 = \Theta(1)$.
The inner product with $\phi_\pi(y^*)$ is:
\[
\langle \phi_\pi(a), \phi_\pi(y^*) \rangle
= \langle e_a - \pi,\, e_{y^*} - \pi \rangle
= -p - \pi(a) + \|\pi\|^2
= -\frac{p(1-p)K}{K{-}1}.
\]
So $\langle \phi_\pi(a), \nabla_z J \rangle = -p^2(1{-}p)\,K/(K{-}1) = -\Theta(p^2)$.
The cosine is $|-\Theta(p^2)| / (\Theta(1) \cdot \Theta(p)) = \Theta(p)$.
Since $\cos(\phi_\pi(a), \nabla_z J) = \Theta(p) \ll 1$, the perpendicular component is $\Theta(1)$.
\end{proof}

\begin{remark}[The arithmetic of noise]
\label{rem:arithmetic}
A batch of $B$ samples contains $\sim pB$ correct draws and $\sim (1-p)B$ incorrect draws.
For a baseline $b$ bounded away from $0$ and $1$, the signal (parallel to $\nabla_z J$) scales as $\Theta(pB)$ and the noise (perpendicular) as $\Theta(\sqrt{B})$; for $b \to 0$ the incorrect-action noise shrinks with $b$.
When $p^2 B \ll 1$: $\cos(\bar g, \nabla_z J) \approx p\sqrt{B}$; the gradient is nearly random unless $B \gg 1/p^2$.
When $p^2 B \gg 1$: signal dominates noise.
\end{remark}

\subsection{Proof of Proposition~\ref{prop:pareto_formal}}

The gate at $\lambda = 0$ skips whenever $\delight < 0$.
Since $\surp > 0$ always, $\delight < 0$ iff $U < 0$ iff the action was incorrect.
The gate keeps only correct-action samples.

\begin{proof}
\medskip\noindent\textbf{Part 1: Direction.}
The gate fires iff $A = y^*$ (probability $p$).
Every kept term is $(1{-}b)\,\phi_\pi(y^*)$.
$\E[g_\mathrm{KG}] = p(1{-}b)\,\phi_\pi(y^*) = (1{-}b)\,\nabla_z J$.
Under normalized steps, the scale $(1{-}b)$ does not affect direction.

\medskip\noindent\textbf{Part 2: Variance.}
Every kept term is the same vector $(1{-}b)\,\phi_\pi(y^*)$.
Every kept term lies along $\nabla_z J$, so $\Var_\perp = 0$.
For PG, each incorrect sample contributes $\|\Pi_\perp(-b\,\phi_\pi(a))\|^2 = b^2 \cdot \Theta(1)$, arriving with probability $1 - p$.
Per-sample: $\Var_\perp(g_\mathrm{PG}) = (1{-}p) \cdot b^2 \cdot \Theta(1)$.

\medskip\noindent\textbf{Part 3: Cost.}
$\Pr(\text{gate fires}) = p$.
Expected backward passes: $pB$.

\medskip\noindent\textbf{Part 4: Alignment.}
The KG batch gradient (given at least one correct draw) is $\bar g_\mathrm{KG} = (1{-}b)\,\phi_\pi(y^*)$, deterministic in direction.
$\cos(\bar g_\mathrm{KG}, \nabla_z J) = 1$.
For $b$ bounded away from zero, PG's batch cosine is $\Theta(p\sqrt{B})$ (Remark~\ref{rem:arithmetic}), requiring $B = \Omega(1/p^2)$ to approach $1$.
KG needs at least one correct draw, which occurs with probability $1-(1-p)^B \ge 1-\delta$ once $B \ge \log(1/\delta)/p$.
\end{proof}

\subsection{Proof of Proposition~\ref{prop:delight_dominance}}

We separate sign consistency (Part~1) from the additive failure mode (Part~2).

\begin{proof}
\medskip\noindent\textbf{Part 1: Sign consistency.}
$\delight(a) = U(a) \cdot \surp(a)$.
$\surp(a) = -\log \pi(a) > 0$ for any $\pi(a) < 1$.
So $\operatorname{sgn}(\delight) = \operatorname{sgn}(U)$.
$U(y^*) = 1 - b > 0$; $U(a \neq y^*) = -b < 0$.

\medskip\noindent\textbf{Part 2: Additive failure.}
For any baseline $b$: $U(y^*) = 1-b$, $U(a) = -b$, $\surp(y^*) = -\log p$, $\surp(a) = \log((K{-}1)/(1{-}p))$.
The additive scores $f_\alpha = (1-\alpha)U + \alpha\surp$ are:
\begin{align*}
f_\alpha(y^*) &= (1{-}\alpha)(1{-}b) - \alpha\log p, \\
f_\alpha(a) &= -(1{-}\alpha) b + \alpha\log\!\frac{K{-}1}{1{-}p}\,.
\end{align*}
Taking the difference and simplifying:
\begin{align*}
f_\alpha(y^*) - f_\alpha(a)
&= (1{-}\alpha) - \alpha\log\!\frac{p(K{-}1)}{1{-}p}
= (1{-}\alpha) - \alpha\,L.
\end{align*}
When $p \le 1/K$: $L \le 0$ and the difference is positive for every $\alpha \in [0,1)$.
When $p > 1/K$: $L > 0$ and the difference is positive iff $\alpha < 1/(1 + L) = \alpha^*$.

$\alpha^*$ shrinks with both action-space size and policy quality:
\begin{center}
\begin{tabular}{ccc}
\toprule
$(K, p)$ & $L = \log(p(K{-}1)/(1{-}p))$ & $\alpha^*$ \\
\midrule
$(10,\; 0.5)$ & $\log 9 \approx 2.2$ & $0.31$ \\
$(100,\; 0.5)$ & $\log 99 \approx 4.6$ & $0.18$ \\
$(100,\; 0.9)$ & $\log 891 \approx 6.8$ & $0.13$ \\
$(50000,\; 0.5)$ & $\log 49999 \approx 10.8$ & $0.08$ \\
\bottomrule
\end{tabular}
\end{center}
As $K$ or $p$ grows, $\alpha$ must approach $0$ (pure advantage), losing the information axis of delight.
\end{proof}

\subsection{Proof of Proposition~\ref{prop:gambling}}
\label{sec:gambling_proof}

We extend the bandit to stochastic rewards: arm~1 pays $\mu^*$ deterministically, arm~2 pays $R_2 \sim \mathcal{N}(\mu^* - \Delta, \sigma^2)$.
Policy $\pi(1) = 1 - \varepsilon$; baseline $b = V^\pi = \mu^* - \varepsilon\Delta$.

\begin{proof}
\medskip\noindent\textbf{Part 1: Reliable regime.}
$U_2 \mid A{=}2$ is Gaussian with mean $-(1{-}\varepsilon)\Delta$ and variance $\sigma^2$.
By the Gaussian tail bound:
$\Pr(U_2 > 0 \mid A{=}2) = \Pr(R_2 > b) \le \exp(-(1{-}\varepsilon)^2 \Delta^2 / (2\sigma^2))$.

\medskip\noindent\textbf{Part 2: Pathological regime.}
$\Pr(U_2 > 0 \mid A{=}2) = 1 - \Phi\big((1{-}\varepsilon)\Delta/\sigma\big)$.
When $\sigma \gg \Delta$, the argument is $o(1)$, so this probability is bounded away from $0$: $\Pr(U_2 > 0 \mid A{=}2) = \Theta(1)$.

\medskip\noindent\textbf{Part 3: Amplification.}
$\{U_2 > 0\} = \{R_2 > b\}$ depends only on the reward distribution, not the priority signal.
Given $U_2 > 0$, advantage-priority assigns weight $|U_2|$, while delight assigns $|U_2| \cdot \surp_2$.
Since $\surp_2 = \log(1/\varepsilon) \to \infty$ as $\varepsilon \to 0$, delight inflates the false positive by a factor that grows as the policy improves.
\end{proof}

\begin{remark}[Limits of per-sample screening]
\label{rem:limit}
No per-sample statistic computed from $(R, \pi)$ can distinguish a genuine breakthrough from a lucky draw.
The distinction requires either multi-trial statistics or a learned reward model.
The gambling pathology is a property of the reward structure ($\sigma/\Delta \gg 1$), and every advantage-based gate inherits it.
\end{remark}

\section{Token Reversal}
\label{app:token_reversal}

We provide experimental details and supplementary scaling metrics for the token reversal experiments in Section~\ref{sec:results}.

\subsection{Experimental Details}
\label{app:token_reversal_details}

\paragraph{Architecture.}
The agent is a decoder-only Transformer with causal attention, model dimension $d_{\text{model}} = 64$, 3 layers, 4 attention heads, feed-forward dimension 128, and learned position embeddings (about 100K parameters), as in \citet{osband2026delightfulpolicygradient}.

\paragraph{Environment.}
The token reversal task presents a prompt of $H$ tokens drawn uniformly from a vocabulary of size $M$; the agent must output the tokens in reverse order.
The per-episode reward is $R = N/H \in [0, 1]$, where $N$ is the number of consecutive correct tokens before the first error.

\paragraph{Baseline.}
All methods use a grouped empirical baseline: each batch consists of $P = 10$ prompts with $S = 10$ sampled responses each, and the baseline for each prompt is the mean reward across its responses.
This is the group-mean baseline of GRPO~\citep{shao2024deepseekmath} without its standard-deviation normalization; any other value function could be substituted.

\paragraph{Optimization.}
All methods use Adam with learning rate $10^{-4}$, except PG, whose learning rate is tuned.
Training runs for $K$ gradient steps with batch size $P \times S = 100$ episodes per step.
For the learning curve experiments (Figure~\ref{fig:token_regret}), $K = 3{,}000$; for the scaling sweeps (Figures~\ref{fig:scaling_vocab}--\ref{fig:scaling_length}), $K = 10{,}000$.
All figures average over 10 seeds.

\paragraph{Kondo gate.}
We test the Kondo gate at fixed gate rates $\rho \in \{0.03, 0.05, 0.1, 0.2, 0.5, 1.0\}$ and in adaptive mode ($\lambda = 0$, variable $\rho$).
The priority signal for screening is delight.

\paragraph{Baselines.}
We compare against three standard reinforcement learning methods within our codebase:
PG (importance-weighted REINFORCE with no clipping), PPO~\citep{schulman2017proximal} with $\epsilon = 0.2$ and $\beta_{\text{KL}} = 0$, and PMPO~\citep{abdolmaleki2024preference} with $\alpha = 1$ and $\beta_{\text{KL}} = 0$.
All baselines use identical architecture, optimizer, and grouped baseline.

\paragraph{Scaling protocol.}
For the vocab scaling sweep (Figure~\ref{fig:scaling_vocab}), we fix $H = 10$ and sweep $M \in \{2, 3, \ldots, 30\}$.
For the length scaling sweep (Figure~\ref{fig:scaling_length}), we fix $M = 2$ and sweep $H \in \{2, 3, \ldots, 30\}$.
A problem is considered solved once the average reward over training exceeds $0.75$ ($0.7$ for the vocabulary sweep), and $H^*$ and $M^*$ never decrease.
The main-body figures report $M^*$ (largest $M$ solved) and $H^*$ (largest $H$ solved) as functions of forward or backward compute.

\subsection{Average Error Scaling}
\label{app:average_error}

The main-body scaling figures (Figures~\ref{fig:scaling_vocab}--\ref{fig:scaling_length}) report the largest problem solved as a function of compute.
A complementary view is the average error across all problem sizes at a fixed compute budget, which captures how gracefully each method degrades as problems exceed its capacity.

Figure~\ref{fig:length_ave_error} plots average error against sequence length $H$ on a log--log scale.
All methods follow roughly linear trends on log--log axes, and the DG family (DG, DG-K) occupies a uniformly lower curve than PG, PPO, and PMPO.
In backward-pass space~(b), DG-K ($\rho = 3\%$) separates dramatically from the pack, achieving the same average error as full DG at a small fraction of the backward cost.

\begin{figure}[ht!]
\centering
\begin{subfigure}[t]{0.48\columnwidth}
    \centering
    \includegraphics[width=\linewidth]{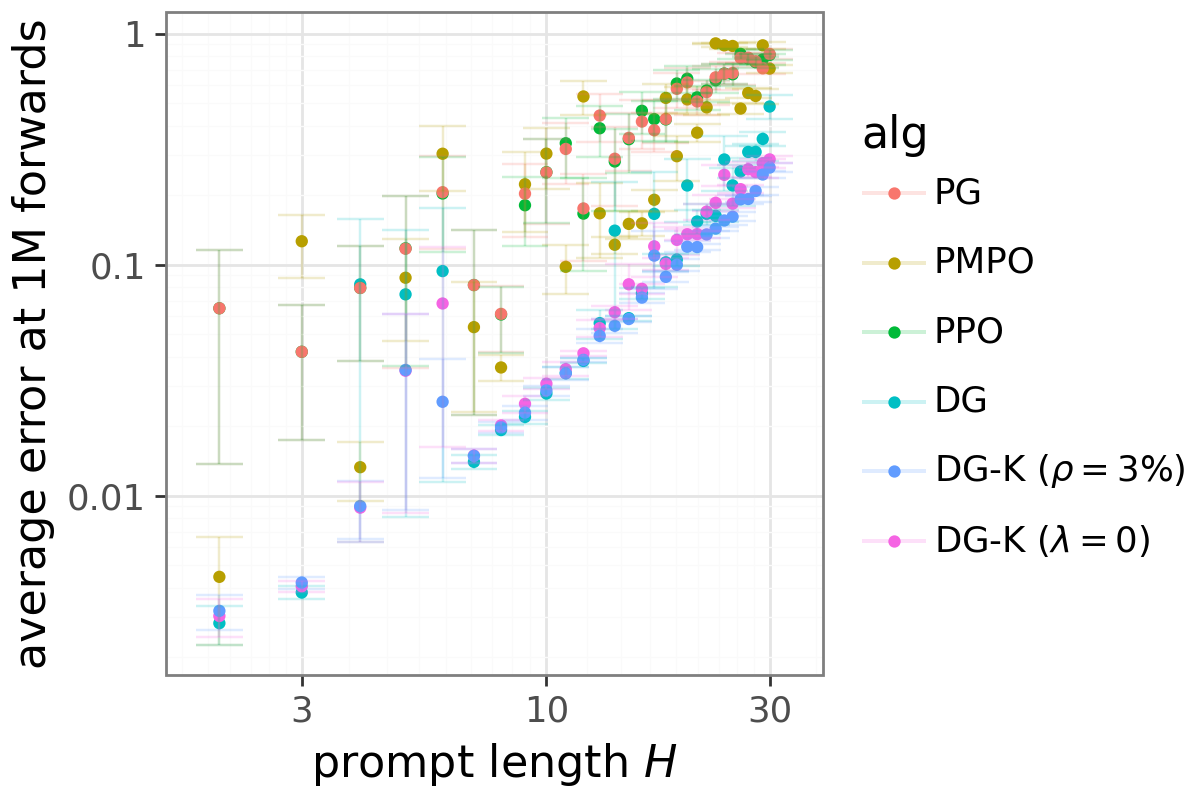}
    \caption{Forward passes.}
\end{subfigure}
\hfill
\begin{subfigure}[t]{0.48\columnwidth}
    \centering
    \includegraphics[width=\linewidth]{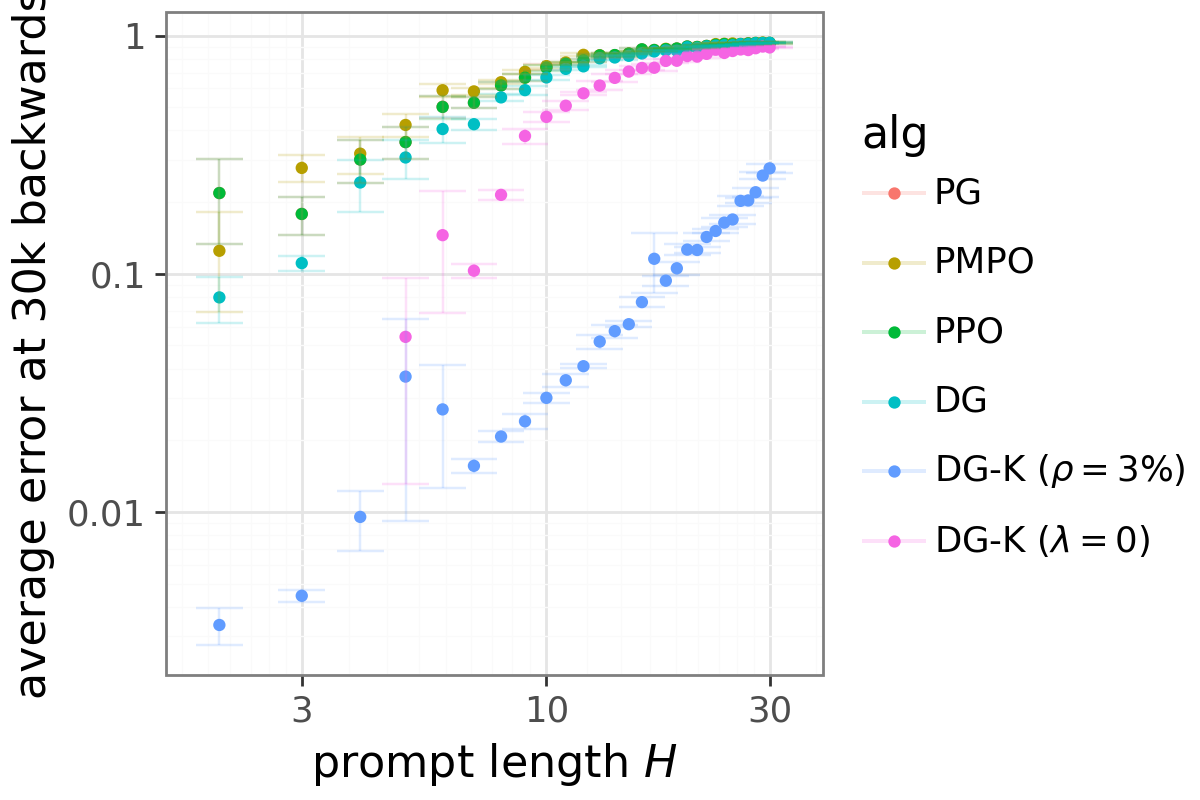}
    \caption{Backward passes.}
\end{subfigure}
\caption{Average error vs.\ sequence length $H$ (log--log, $M = 2$).
DG-K compresses the backward axis.}
\label{fig:length_ave_error}
\end{figure}

Figure~\ref{fig:vocab_ave_error} shows the same comparison over vocabulary size $M$.
As with sequence length, DG and DG-K ($\lambda = 0$) trace the lowest curves in forward-pass space, while DG-K ($\rho = 3\%$) is best in backward-pass space.
As $M$ grows and informative tokens become rarer, the gate's backward-pass advantage widens.

\begin{figure}[ht!]
\centering
\begin{subfigure}[t]{0.48\columnwidth}
    \centering
    \includegraphics[width=\linewidth]{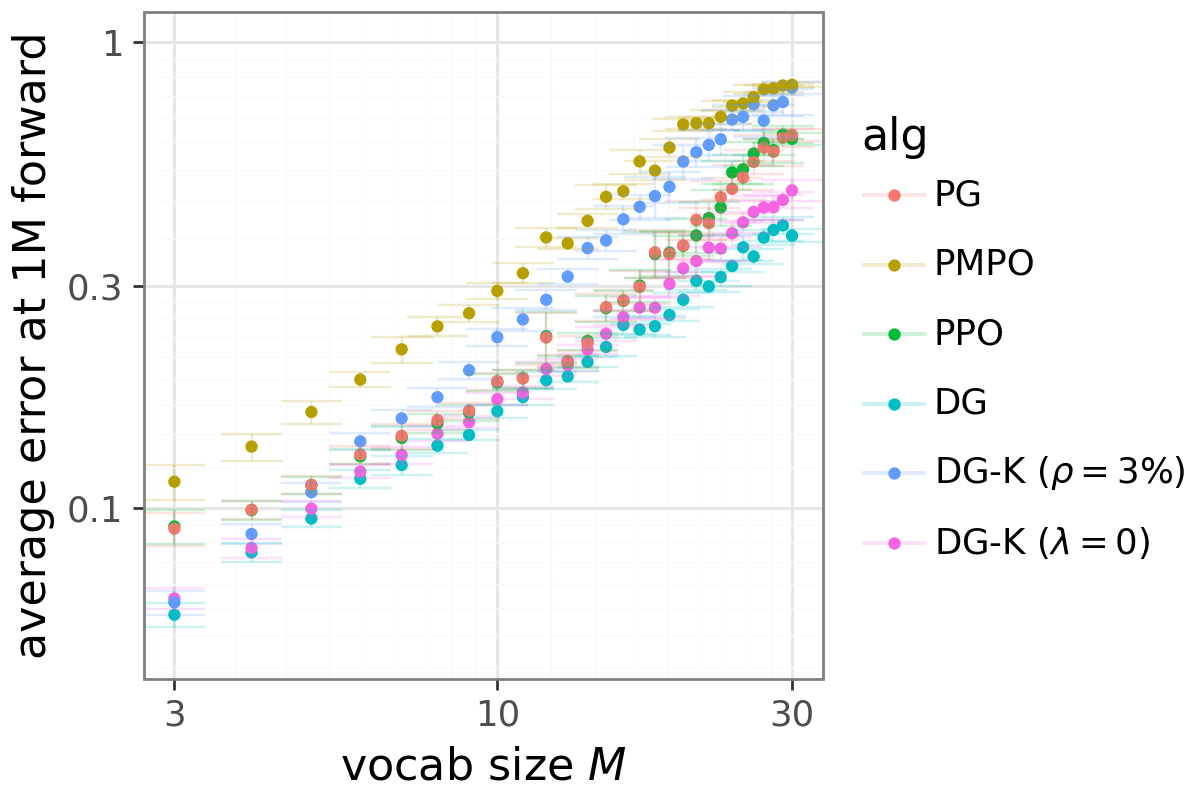}
    \caption{Forward passes.}
\end{subfigure}
\hfill
\begin{subfigure}[t]{0.48\columnwidth}
    \centering
    \includegraphics[width=\linewidth]{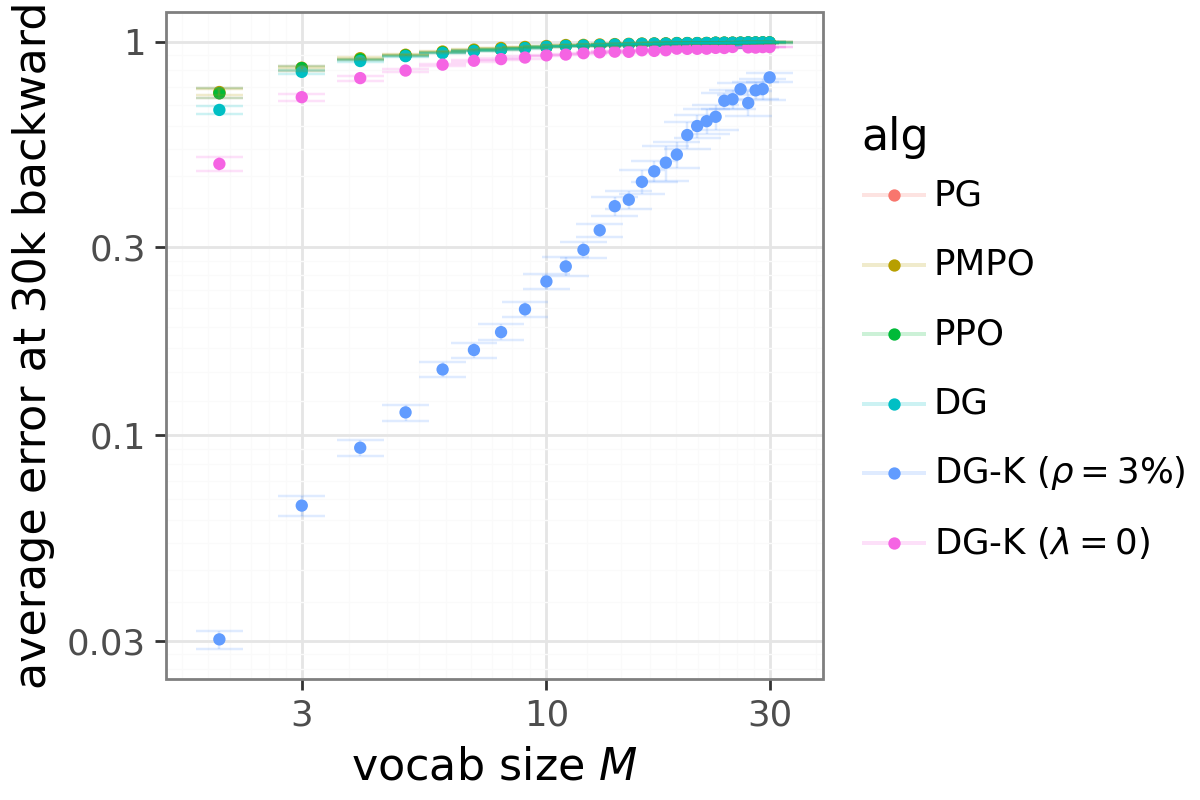}
    \caption{Backward passes.}
\end{subfigure}
\caption{Average error vs.\ vocabulary size $M$ (log--log, $H = 10$).
The gate's backward-pass advantage widens with $M$.}
\label{fig:vocab_ave_error}
\end{figure}

\subsection{Final Error Scaling}
\label{app:final_error}

The average error view aggregates performance across the full training trajectory.
An alternative is to examine the final error at a fixed compute budget, which isolates where each method converges given a finite allocation.

Figure~\ref{fig:length_final} plots final error against sequence length $H$.
DG and both DG-K variants maintain near-zero final error across a wide range of $H$, while PG, PPO, and PMPO degrade steadily.
In backward-pass space~(b), the advantage is stark: DG-K ($\rho = 3\%$) stays near zero across almost the entire range, while baselines climb steeply.

\begin{figure}[ht!]
\centering
\begin{subfigure}[t]{0.48\columnwidth}
    \centering
    \includegraphics[width=\linewidth]{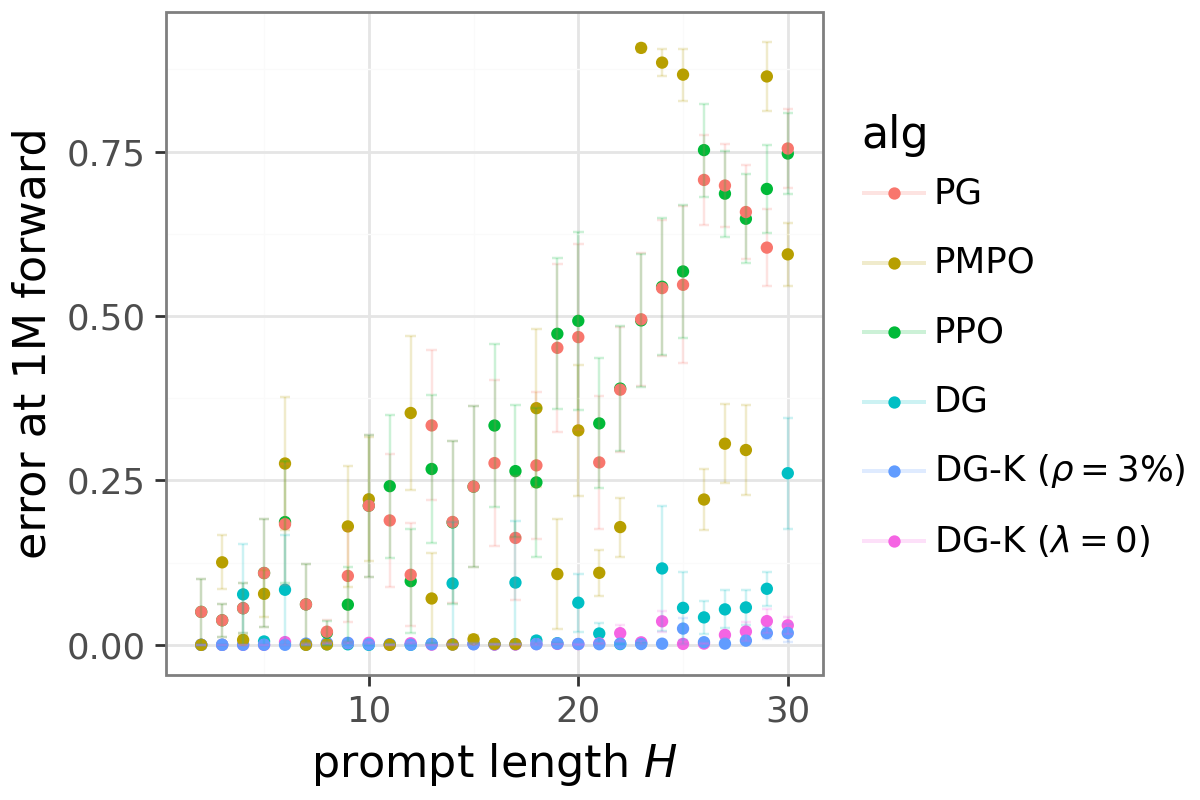}
    \caption{Forward passes.}
\end{subfigure}
\hfill
\begin{subfigure}[t]{0.48\columnwidth}
    \centering
    \includegraphics[width=\linewidth]{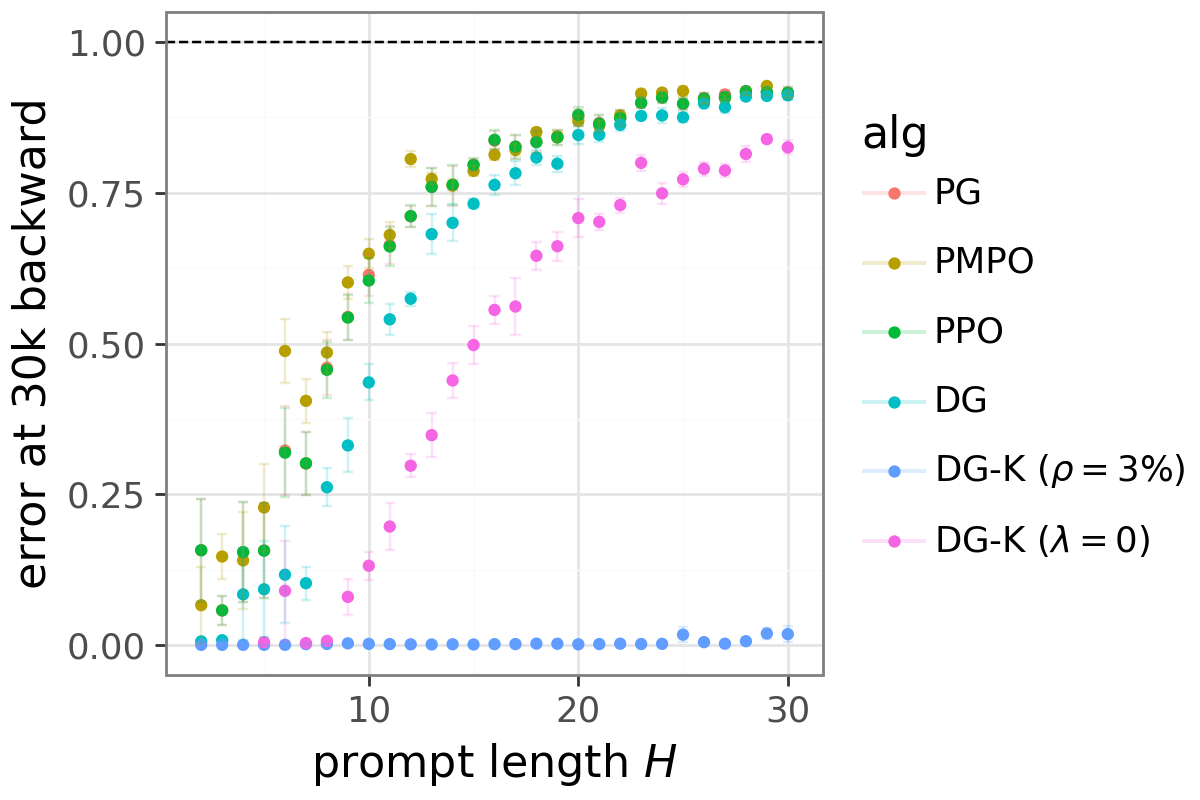}
    \caption{Backward passes.}
\end{subfigure}
\caption{Final error vs.\ sequence length $H$ ($M = 2$).
DG-K stays near zero where baselines degrade.}
\label{fig:length_final}
\end{figure}

Figure~\ref{fig:vocab_final} shows the same over vocabulary size $M$.
The adaptive gate ($\lambda = 0$) tracks full DG faithfully, while the fixed gate ($\rho = 3\%$) degrades at large $M$ where a fixed 3\% budget becomes too aggressive.
In backward-pass space, both Kondo variants still outperform all baselines across the full range.

\begin{figure}[ht!]
\centering
\begin{subfigure}[t]{0.48\columnwidth}
    \centering
    \includegraphics[width=\linewidth]{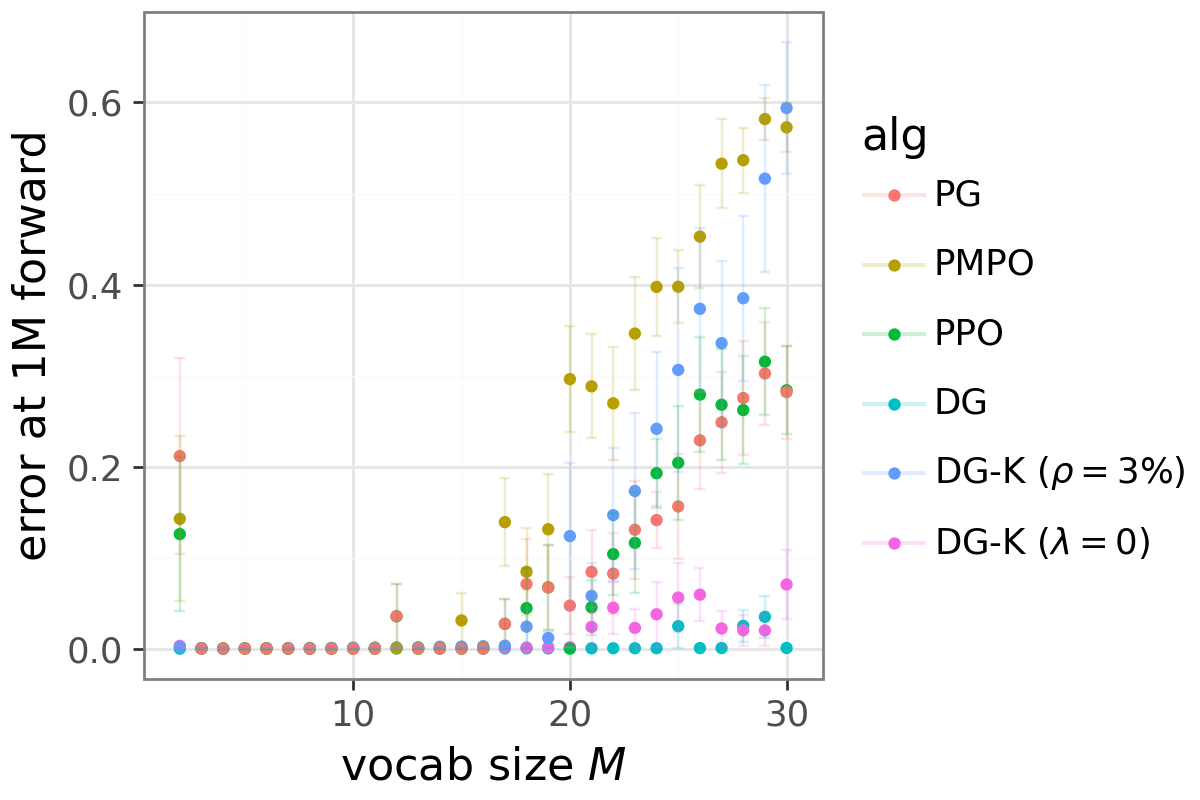}
    \caption{Forward passes.}
\end{subfigure}
\hfill
\begin{subfigure}[t]{0.48\columnwidth}
    \centering
    \includegraphics[width=\linewidth]{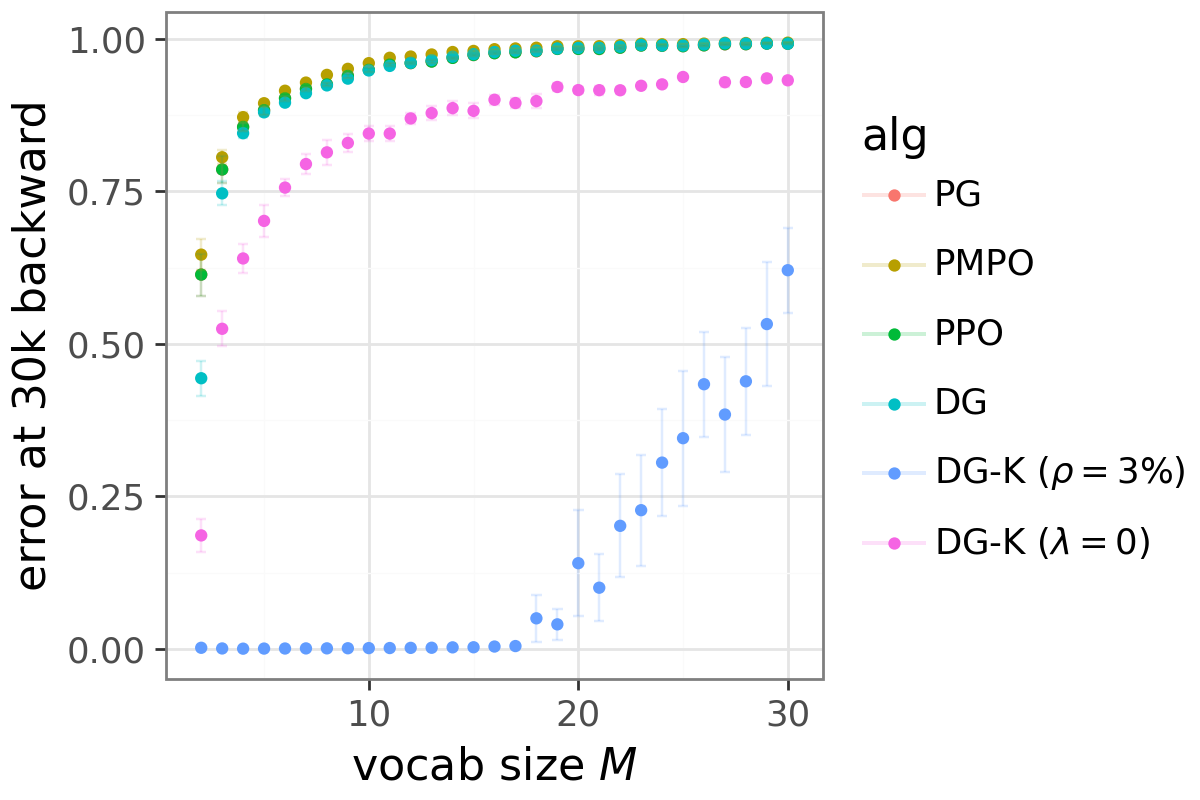}
    \caption{Backward passes.}
\end{subfigure}
\caption{Final error vs.\ vocabulary size $M$ ($H = 10$).
Adaptive gate tracks DG; fixed gate degrades at large $M$ but dominates in backward-pass space.}
\label{fig:vocab_final}
\end{figure}

\section{Limitations and Scope}
\label{app:limitations}

The main-body conclusion states the main limitations: savings counted rather than timed, deliberate bias with open convergence, tasks without preference-style reward ambiguity, the gambling pathology, and scale.
Here we collect additional caveats.

\paragraph{Backward-pass counts vs.\ wall-clock time.}
Our experiments measure backward-pass counts, not accelerator wall-clock time.
This isolates whether learning quality survives sparse backward computation, but it does not imply that a naive loss mask yields proportional wall-clock savings.
Real speedups require implementations that avoid differentiating skipped samples, for example through example-level gating, sequence packing, or a two-stage screening pipeline.

\paragraph{Credit assignment at the token level.}
In sequence-level reward settings, token-level delight inherits the usual credit-assignment ambiguity: a high-surprisal token in a successful response may be incidental rather than causal.
Span-level or trajectory-level gates may be more appropriate in some language-model applications.

\paragraph{Advantage estimation.}
The gate depends on the sign and scale of the advantage $U_t$.
Errors in the baseline can change which samples receive backward passes.
Our MNIST experiments show baseline robustness (Appendix~\ref{app:mnist_baselines}), but in RLHF or long-horizon RL, advantage estimation noise is a central concern.

\paragraph{Negative evidence.}
Positive-delight gating systematically skips negative-advantage samples.
While Proposition~\ref{prop:pareto_formal} shows this is beneficial in the deterministic-reward setting, rare catastrophic failures or newly discovered traps may carry useful negative evidence that the gate discards.
A magnitude-based gate ($|\delight|$) or a two-sided gate could address this, though at the cost of reintroducing some perpendicular noise.

\newpage
\input{checklist.tex}

\end{document}

%% file: macros.tex
\usepackage{amsmath}
\usepackage{amssymb}
\usepackage{algorithm}
\usepackage[noend]{algpseudocode}
\usepackage{mathrsfs}
\usepackage{dsfont}
\usepackage{array}
\usepackage{bbm}
\usepackage[mathscr]{eucal}
\usepackage{psfrag}
\usepackage{here}
\usepackage{wasysym}
\usepackage[dvipsnames, table]{xcolor}  
\usepackage{caption}                     
\usepackage{subcaption}
\usepackage{amsthm}
\usepackage{todonotes}
\usepackage{tabulary}
\usepackage{outlines}
\usepackage{thmtools, thm-restate}

\newtheorem{proposition} {Proposition}
\newtheorem{lemma} {Lemma}

\newtheorem{assumption} {Assumption}

\newtheorem{remark} {Remark}

\newcommand{\E}{\mathbb{E}}

\newcommand{\Var}{\mathrm{Var}}

\newcommand{\surp}{\ell} 
\newcommand{\delight}{\chi} 

\newcommand{\mix}{\eta} 

\definecolor{ian_highlight}{RGB}{100, 2, 2}

\errorcontextlines\maxdimen

\makeatletter
    \newcommand*{\algrule}[1][\algorithmicindent]{\makebox[#1][l]{\hspace*{.5em}\thealgruleextra\vrule height \thealgruleheight depth \thealgruledepth}}%
\newcommand*{\thealgruleextra}{}
\newcommand*{\thealgruleheight}{.75\baselineskip}
\newcommand*{\thealgruledepth}{.25\baselineskip}

\newcount\ALG@printindent@tempcnta
\def\ALG@printindent{%
    \ifnum \theALG@nested>0
        \ifx\ALG@text\ALG@x@notext
        \else
            \unskip
            \addvspace{-1pt}
            \ALG@printindent@tempcnta=1
            \loop
                \algrule[\csname ALG@ind@\the\ALG@printindent@tempcnta\endcsname]%
                \advance \ALG@printindent@tempcnta 1
            \ifnum \ALG@printindent@tempcnta<\numexpr\theALG@nested+1\relax
            \repeat
        \fi
    \fi
    }%
\usepackage{etoolbox}
\makeatother

\newbox\statebox
\newcommand{\myState}[1]{%
    \setbox\statebox=\vbox{#1}%
    \edef\thealgruleheight{\dimexpr \the\ht\statebox+1pt\relax}%
    \edef\thealgruledepth{\dimexpr \the\dp\statebox+1pt\relax}%
    \ifdim\thealgruleheight<.75\baselineskip
        \def\thealgruleheight{\dimexpr .75\baselineskip+1pt\relax}%
    \fi
    \ifdim\thealgruledepth<.25\baselineskip
        \def\thealgruledepth{\dimexpr .25\baselineskip+1pt\relax}%
    \fi
    \State #1%
    \def\thealgruleheight{\dimexpr .75\baselineskip+1pt\relax}%
    \def\thealgruledepth{\dimexpr .25\baselineskip+1pt\relax}%
}

%% file: checklist.tex
\section*{NeurIPS Paper Checklist}

\begin{enumerate}

\item {\bf Claims}
    \item[] Question: Do the main claims made in the abstract and introduction accurately reflect the paper's contributions and scope?
    \item[] Answer: \answerYes{}
    \item[] Justification: The abstract and introduction claim that the Kondo gate skips most backward passes while retaining nearly all of DG's learning quality, with gains that grow as problems get harder. These claims are supported by experiments on MNIST (Section~\ref{sec:mnist}), tabular bandits (Section~\ref{sec:bandit}), and transformer token reversal (Section~\ref{sec:results}). Limitations are stated explicitly in the Conclusion (Section~\ref{sec:conclusion}) and Appendix~\ref{app:limitations}.
    \item[] Guidelines:
    \begin{itemize}
        \item The answer \answerNA{} means that the abstract and introduction do not include the claims made in the paper.
        \item The abstract and/or introduction should clearly state the claims made, including the contributions made in the paper and important assumptions and limitations. A \answerNo{} or \answerNA{} answer to this question will not be perceived well by the reviewers. 
        \item The claims made should match theoretical and experimental results, and reflect how much the results can be expected to generalize to other settings. 
        \item It is fine to include aspirational goals as motivation as long as it is clear that these goals are not attained by the paper. 
    \end{itemize}

\item {\bf Limitations}
    \item[] Question: Does the paper discuss the limitations of the work performed by the authors?
    \item[] Answer: \answerYes{}
    \item[] Justification: The Conclusion (Section~\ref{sec:conclusion}) states the main limitations: savings are counted in selected gradient terms rather than wall-clock time, the gated estimator is deliberately biased with general convergence open, the tasks lack the reward ambiguity of preference-based learning, and the gambling pathology (Proposition~\ref{prop:gambling}) marks where delight misleads. Appendix~\ref{app:limitations} provides extended discussion of additional caveats including credit assignment at the token level, advantage estimation sensitivity, and the systematic exclusion of negative evidence. The scale question is identified as the central open problem.
    \item[] Guidelines:
    \begin{itemize}
        \item The answer \answerNA{} means that the paper has no limitation while the answer \answerNo{} means that the paper has limitations, but those are not discussed in the paper. 
        \item The authors are encouraged to create a separate ``Limitations'' section in their paper.
        \item The paper should point out any strong assumptions and how robust the results are to violations of these assumptions (e.g., independence assumptions, noiseless settings, model well-specification, asymptotic approximations only holding locally). The authors should reflect on how these assumptions might be violated in practice and what the implications would be.
        \item The authors should reflect on the scope of the claims made, e.g., if the approach was only tested on a few datasets or with a few runs. In general, empirical results often depend on implicit assumptions, which should be articulated.
        \item The authors should reflect on the factors that influence the performance of the approach. For example, a facial recognition algorithm may perform poorly when image resolution is low or images are taken in low lighting. Or a speech-to-text system might not be used reliably to provide closed captions for online lectures because it fails to handle technical jargon.
        \item The authors should discuss the computational efficiency of the proposed algorithms and how they scale with dataset size.
        \item If applicable, the authors should discuss possible limitations of their approach to address problems of privacy and fairness.
        \item While the authors might fear that complete honesty about limitations might be used by reviewers as grounds for rejection, a worse outcome might be that reviewers discover limitations that aren't acknowledged in the paper. The authors should use their best judgment and recognize that individual actions in favor of transparency play an important role in developing norms that preserve the integrity of the community. Reviewers will be specifically instructed to not penalize honesty concerning limitations.
    \end{itemize}

\item {\bf Theory assumptions and proofs}
    \item[] Question: For each theoretical result, does the paper provide the full set of assumptions and a complete (and correct) proof?
    \item[] Answer: \answerYes{}
    \item[] Justification: Propositions~\ref{prop:pareto_formal} and~\ref{prop:delight_dominance} use Assumption~\ref{assump:bandit}, stated in full in Appendix~\ref{app:proofs}; Proposition~\ref{prop:gambling} states its own two-armed stochastic model. Complete proofs are provided in the appendix with labeled paragraphs for each part. The gate derivation appears in Appendix~\ref{app:gate_derivation}.
    \item[] Guidelines:
    \begin{itemize}
        \item The answer \answerNA{} means that the paper does not include theoretical results. 
        \item All the theorems, formulas, and proofs in the paper should be numbered and cross-referenced.
        \item All assumptions should be clearly stated or referenced in the statement of any theorems.
        \item The proofs can either appear in the main paper or the supplemental material, but if they appear in the supplemental material, the authors are encouraged to provide a short proof sketch to provide intuition. 
        \item Inversely, any informal proof provided in the core of the paper should be complemented by formal proofs provided in appendix or supplemental material.
        \item Theorems and Lemmas that the proof relies upon should be properly referenced. 
    \end{itemize}

    \item {\bf Experimental result reproducibility}
    \item[] Question: Does the paper fully disclose all the information needed to reproduce the main experimental results of the paper to the extent that it affects the main claims and/or conclusions of the paper (regardless of whether the code and data are provided or not)?
    \item[] Answer: \answerYes{}
    \item[] Justification: Full experimental details are provided: MNIST setup in Appendix~\ref{app:mnist_details} (architecture, optimizer, learning rate sweeps, baseline, gate rate sweeps) and token reversal in Appendix~\ref{app:token_reversal_details} (architecture, environment, baseline, optimizer, gate configurations, scaling protocol). Algorithm~\ref{alg:kondo} specifies the method precisely. All hyperparameter values and sweep ranges are listed.
    \item[] Guidelines:
    \begin{itemize}
        \item The answer \answerNA{} means that the paper does not include experiments.
        \item If the paper includes experiments, a \answerNo{} answer to this question will not be perceived well by the reviewers: Making the paper reproducible is important, regardless of whether the code and data are provided or not.
        \item If the contribution is a dataset and\slash or model, the authors should describe the steps taken to make their results reproducible or verifiable. 
        \item Depending on the contribution, reproducibility can be accomplished in various ways. For example, if the contribution is a novel architecture, describing the architecture fully might suffice, or if the contribution is a specific model and empirical evaluation, it may be necessary to either make it possible for others to replicate the model with the same dataset, or provide access to the model. In general. releasing code and data is often one good way to accomplish this, but reproducibility can also be provided via detailed instructions for how to replicate the results, access to a hosted model (e.g., in the case of a large language model), releasing of a model checkpoint, or other means that are appropriate to the research performed.
        \item While NeurIPS does not require releasing code, the conference does require all submissions to provide some reasonable avenue for reproducibility, which may depend on the nature of the contribution. For example
        \begin{enumerate}
            \item If the contribution is primarily a new algorithm, the paper should make it clear how to reproduce that algorithm.
            \item If the contribution is primarily a new model architecture, the paper should describe the architecture clearly and fully.
            \item If the contribution is a new model (e.g., a large language model), then there should either be a way to access this model for reproducing the results or a way to reproduce the model (e.g., with an open-source dataset or instructions for how to construct the dataset).
            \item We recognize that reproducibility may be tricky in some cases, in which case authors are welcome to describe the particular way they provide for reproducibility. In the case of closed-source models, it may be that access to the model is limited in some way (e.g., to registered users), but it should be possible for other researchers to have some path to reproducing or verifying the results.
        \end{enumerate}
    \end{itemize}

\item {\bf Open access to data and code}
    \item[] Question: Does the paper provide open access to the data and code, with sufficient instructions to faithfully reproduce the main experimental results, as described in supplemental material?
    \item[] Answer: \answerYes{}
    \item[] Justification: Code is available at \url{https://github.com/google-deepmind/egg}. The experiments use MNIST (publicly available) and a synthetic token reversal task fully specified in Appendix~\ref{app:token_reversal_details}. No proprietary data is required.
    \item[] Guidelines:
    \begin{itemize}
        \item The answer \answerNA{} means that paper does not include experiments requiring code.
        \item Please see the NeurIPS code and data submission guidelines (\url{https://neurips.cc/public/guides/CodeSubmissionPolicy}) for more details.
        \item While we encourage the release of code and data, we understand that this might not be possible, so \answerNo{} is an acceptable answer. Papers cannot be rejected simply for not including code, unless this is central to the contribution (e.g., for a new open-source benchmark).
        \item The instructions should contain the exact command and environment needed to run to reproduce the results. See the NeurIPS code and data submission guidelines (\url{https://neurips.cc/public/guides/CodeSubmissionPolicy}) for more details.
        \item The authors should provide instructions on data access and preparation, including how to access the raw data, preprocessed data, intermediate data, and generated data, etc.
        \item The authors should provide scripts to reproduce all experimental results for the new proposed method and baselines. If only a subset of experiments are reproducible, they should state which ones are omitted from the script and why.
        \item At submission time, to preserve anonymity, the authors should release anonymized versions (if applicable).
        \item Providing as much information as possible in supplemental material (appended to the paper) is recommended, but including URLs to data and code is permitted.
    \end{itemize}

\item {\bf Experimental setting/details}
    \item[] Question: Does the paper specify all the training and test details (e.g., data splits, hyperparameters, how they were chosen, type of optimizer) necessary to understand the results?
    \item[] Answer: \answerYes{}
    \item[] Justification: Appendix~\ref{app:mnist_details} specifies the MNIST architecture, optimizer (Adam), learning rate sweep, batch size, number of gradient steps, validation schedule, baseline, and gate rate sweep. Appendix~\ref{app:token_reversal_details} specifies the transformer architecture, environment definition, baseline (grouped empirical, analogous to GRPO), optimizer, training length, gate configurations, and scaling protocol. All seed counts are reported.
    \item[] Guidelines:
    \begin{itemize}
        \item The answer \answerNA{} means that the paper does not include experiments.
        \item The experimental setting should be presented in the core of the paper to a level of detail that is necessary to appreciate the results and make sense of them.
        \item The full details can be provided either with the code, in appendix, or as supplemental material.
    \end{itemize}

\item {\bf Experiment statistical significance}
    \item[] Question: Does the paper report error bars suitably and correctly defined or other appropriate information about the statistical significance of the experiments?
    \item[] Answer: \answerYes{}
    \item[] Justification: All MNIST experiments average over 30 seeds with $\pm 1$ standard error shading (Section~\ref{sec:core_results}). All token reversal experiments average over 10 seeds with $\pm 1$ standard error shading (Section~\ref{sec:results}). Error bars are clearly identified as standard error of the mean in figure captions.
    \item[] Guidelines:
    \begin{itemize}
        \item The answer \answerNA{} means that the paper does not include experiments.
        \item The authors should answer \answerYes{} if the results are accompanied by error bars, confidence intervals, or statistical significance tests, at least for the experiments that support the main claims of the paper.
        \item The factors of variability that the error bars are capturing should be clearly stated (for example, train/test split, initialization, random drawing of some parameter, or overall run with given experimental conditions).
        \item The method for calculating the error bars should be explained (closed form formula, call to a library function, bootstrap, etc.)
        \item The assumptions made should be given (e.g., Normally distributed errors).
        \item It should be clear whether the error bar is the standard deviation or the standard error of the mean.
        \item It is OK to report 1-sigma error bars, but one should state it. The authors should preferably report a 2-sigma error bar than state that they have a 96\% CI, if the hypothesis of Normality of errors is not verified.
        \item For asymmetric distributions, the authors should be careful not to show in tables or figures symmetric error bars that would yield results that are out of range (e.g., negative error rates).
        \item If error bars are reported in tables or plots, the authors should explain in the text how they were calculated and reference the corresponding figures or tables in the text.
    \end{itemize}

\item {\bf Experiments compute resources}
    \item[] Question: For each experiment, does the paper provide sufficient information on the computer resources (type of compute workers, memory, time of execution) needed to reproduce the experiments?
    \item[] Answer: \answerNo{}
    \item[] Justification: Specific hardware types and wall-clock times are not reported. All experiments are small-scale (two-layer MLP on MNIST, 3-layer transformer on token reversal) and run on a single CPU or GPU in minutes to hours.
    \item[] Guidelines:
    \begin{itemize}
        \item The answer \answerNA{} means that the paper does not include experiments.
        \item The paper should indicate the type of compute workers CPU or GPU, internal cluster, or cloud provider, including relevant memory and storage.
        \item The paper should provide the amount of compute required for each of the individual experimental runs as well as estimate the total compute. 
        \item The paper should disclose whether the full research project required more compute than the experiments reported in the paper (e.g., preliminary or failed experiments that didn't make it into the paper). 
    \end{itemize}
    
\item {\bf Code of ethics}
    \item[] Question: Does the research conducted in the paper conform, in every respect, with the NeurIPS Code of Ethics \url{https://neurips.cc/public/EthicsGuidelines}?
    \item[] Answer: \answerYes{}
    \item[] Justification: The research uses only publicly available data (MNIST) and synthetic environments (token reversal). No human subjects, private data, or dual-use concerns are involved.
    \item[] Guidelines:
    \begin{itemize}
        \item The answer \answerNA{} means that the authors have not reviewed the NeurIPS Code of Ethics.
        \item If the authors answer \answerNo, they should explain the special circumstances that require a deviation from the Code of Ethics.
        \item The authors should make sure to preserve anonymity (e.g., if there is a special consideration due to laws or regulations in their jurisdiction).
    \end{itemize}

\item {\bf Broader impacts}
    \item[] Question: Does the paper discuss both potential positive societal impacts and negative societal impacts of the work performed?
    \item[] Answer: \answerNA{}
    \item[] Justification: This is foundational research on compute-efficient gradient estimation for reinforcement learning. The method reduces training compute, which has a positive environmental impact. There is no direct path to negative applications beyond what applies to any generic improvement in neural network training efficiency.
    \item[] Guidelines:
    \begin{itemize}
        \item The answer \answerNA{} means that there is no societal impact of the work performed.
        \item If the authors answer \answerNA{} or \answerNo, they should explain why their work has no societal impact or why the paper does not address societal impact.
        \item Examples of negative societal impacts include potential malicious or unintended uses (e.g., disinformation, generating fake profiles, surveillance), fairness considerations (e.g., deployment of technologies that could make decisions that unfairly impact specific groups), privacy considerations, and security considerations.
        \item The conference expects that many papers will be foundational research and not tied to particular applications, let alone deployments. However, if there is a direct path to any negative applications, the authors should point it out. For example, it is legitimate to point out that an improvement in the quality of generative models could be used to generate Deepfakes for disinformation. On the other hand, it is not needed to point out that a generic algorithm for optimizing neural networks could enable people to train models that generate Deepfakes faster.
        \item The authors should consider possible harms that could arise when the technology is being used as intended and functioning correctly, harms that could arise when the technology is being used as intended but gives incorrect results, and harms following from (intentional or unintentional) misuse of the technology.
        \item If there are negative societal impacts, the authors could also discuss possible mitigation strategies (e.g., gated release of models, providing defenses in addition to attacks, mechanisms for monitoring misuse, mechanisms to monitor how a system learns from feedback over time, improving the efficiency and accessibility of ML).
    \end{itemize}
    
\item {\bf Safeguards}
    \item[] Question: Does the paper describe safeguards that have been put in place for responsible release of data or models that have a high risk for misuse (e.g., pre-trained language models, image generators, or scraped datasets)?
    \item[] Answer: \answerNA{}
    \item[] Justification: The paper does not release pre-trained models, scraped datasets, or any artifacts with high misuse risk. The released code implements a gradient-gating algorithm for small-scale RL experiments.
    \item[] Guidelines:
    \begin{itemize}
        \item The answer \answerNA{} means that the paper poses no such risks.
        \item Released models that have a high risk for misuse or dual-use should be released with necessary safeguards to allow for controlled use of the model, for example by requiring that users adhere to usage guidelines or restrictions to access the model or implementing safety filters. 
        \item Datasets that have been scraped from the Internet could pose safety risks. The authors should describe how they avoided releasing unsafe images.
        \item We recognize that providing effective safeguards is challenging, and many papers do not require this, but we encourage authors to take this into account and make a best faith effort.
    \end{itemize}

\item {\bf Licenses for existing assets}
    \item[] Question: Are the creators or original owners of assets (e.g., code, data, models), used in the paper, properly credited and are the license and terms of use explicitly mentioned and properly respected?
    \item[] Answer: \answerYes{}
    \item[] Justification: MNIST is cited via \citet{lecun1998gradient}. All baseline methods (PG, PPO, PMPO, DG) are cited to their original papers. The codebase is released under the license specified in the repository.
    \item[] Guidelines:
    \begin{itemize}
        \item The answer \answerNA{} means that the paper does not use existing assets.
        \item The authors should cite the original paper that produced the code package or dataset.
        \item The authors should state which version of the asset is used and, if possible, include a URL.
        \item The name of the license (e.g., CC-BY 4.0) should be included for each asset.
        \item For scraped data from a particular source (e.g., website), the copyright and terms of service of that source should be provided.
        \item If assets are released, the license, copyright information, and terms of use in the package should be provided. For popular datasets, \url{paperswithcode.com/datasets} has curated licenses for some datasets. Their licensing guide can help determine the license of a dataset.
        \item For existing datasets that are re-packaged, both the original license and the license of the derived asset (if it has changed) should be provided.
        \item If this information is not available online, the authors are encouraged to reach out to the asset's creators.
    \end{itemize}

\item {\bf New assets}
    \item[] Question: Are new assets introduced in the paper well documented and is the documentation provided alongside the assets?
    \item[] Answer: \answerYes{}
    \item[] Justification: The code release includes the Kondo gate implementation, all experimental configurations, and instructions for reproducing the results. The token reversal environment is fully specified in Appendix~\ref{app:token_reversal_details}.
    \item[] Guidelines:
    \begin{itemize}
        \item The answer \answerNA{} means that the paper does not release new assets.
        \item Researchers should communicate the details of the dataset\slash code\slash model as part of their submissions via structured templates. This includes details about training, license, limitations, etc. 
        \item The paper should discuss whether and how consent was obtained from people whose asset is used.
        \item At submission time, remember to anonymize your assets (if applicable). You can either create an anonymized URL or include an anonymized zip file.
    \end{itemize}

\item {\bf Crowdsourcing and research with human subjects}
    \item[] Question: For crowdsourcing experiments and research with human subjects, does the paper include the full text of instructions given to participants and screenshots, if applicable, as well as details about compensation (if any)? 
    \item[] Answer: \answerNA{}
    \item[] Justification: The paper does not involve crowdsourcing or research with human subjects.
    \item[] Guidelines:
    \begin{itemize}
        \item The answer \answerNA{} means that the paper does not involve crowdsourcing nor research with human subjects.
        \item Including this information in the supplemental material is fine, but if the main contribution of the paper involves human subjects, then as much detail as possible should be included in the main paper. 
        \item According to the NeurIPS Code of Ethics, workers involved in data collection, curation, or other labor should be paid at least the minimum wage in the country of the data collector. 
    \end{itemize}

\item {\bf Institutional review board (IRB) approvals or equivalent for research with human subjects}
    \item[] Question: Does the paper describe potential risks incurred by study participants, whether such risks were disclosed to the subjects, and whether Institutional Review Board (IRB) approvals (or an equivalent approval/review based on the requirements of your country or institution) were obtained?
    \item[] Answer: \answerNA{}
    \item[] Justification: The paper does not involve research with human subjects.
    \item[] Guidelines:
    \begin{itemize}
        \item The answer \answerNA{} means that the paper does not involve crowdsourcing nor research with human subjects.
        \item Depending on the country in which research is conducted, IRB approval (or equivalent) may be required for any human subjects research. If you obtained IRB approval, you should clearly state this in the paper. 
        \item We recognize that the procedures for this may vary significantly between institutions and locations, and we expect authors to adhere to the NeurIPS Code of Ethics and the guidelines for their institution. 
        \item For initial submissions, do not include any information that would break anonymity (if applicable), such as the institution conducting the review.
    \end{itemize}

\item {\bf Declaration of LLM usage}
    \item[] Question: Does the paper describe the usage of LLMs if it is an important, original, or non-standard component of the core methods in this research? Note that if the LLM is used only for writing, editing, or formatting purposes and does \emph{not} impact the core methodology, scientific rigor, or originality of the research, declaration is not required.
    \item[] Answer: \answerNA{}
    \item[] Justification: LLMs are not used as a component of the core method. The Kondo gate is a gradient-gating mechanism applied to standard policy gradient updates; no LLM is involved in the algorithm or experiments.
    \item[] Guidelines:
    \begin{itemize}
        \item The answer \answerNA{} means that the core method development in this research does not involve LLMs as any important, original, or non-standard components.
        \item Please refer to our LLM policy in the NeurIPS handbook for what should or should not be described.
    \end{itemize}

\end{enumerate}

%% file: references.bib
@article{abdolmaleki2024preference,
  author  = {Abdolmaleki, Abbas and Piot, Bilal and Shahriari, Bobak and Springenberg, Jost Tobias and Hertweck, Tim and Joshi, Rishabh and Oh, Junhyuk and Bloesch, Michael and Lampe, Thomas and Heess, Nicolas and others},
  title   = {Preference optimization as probabilistic inference},
  journal = {arXiv preprint arXiv:2410.04166},
  year    = {2024}
}

@inproceedings{katharopoulos2018not,
  author    = {Katharopoulos, Angelos and Fleuret, Fran\c{c}ois},
  title     = {Not All Samples Are Created Equal: Deep Learning with Importance Sampling},
  booktitle = {International Conference on Machine Learning},
  pages     = {2525--2534},
  year      = {2018}
}

@inproceedings{leviathan2023fast,
  author    = {Leviathan, Yaniv and Kalman, Matan and Matias, Yossi},
  title     = {Fast Inference from Transformers via Speculative Decoding},
  booktitle = {International Conference on Machine Learning},
  pages     = {19274--19286},
  year      = {2023}
}

@article{osband2026delightfulpolicygradient,
  title   = {Delightful Policy Gradient},
  author  = {Ian Osband},
  journal = {Reinforcement Learning Journal},
  volume  = {7},
  year    = {2026}
}

@inproceedings{osband2026delightfuldistributedpolicygradient,
  title     = {Delightful Policy Gradient under Distributed Friction},
  author    = {Ian Osband},
  booktitle = {Advances in Neural Information Processing Systems},
  year      = {2026}
}

@article{peng2019advantage,
  author  = {Peng, Xue Bin and Kumar, Aviral and Zhang, Grace and Levine, Sergey},
  title   = {Advantage-weighted regression: Simple and scalable off-policy reinforcement learning},
  journal = {arXiv preprint arXiv:1910.00177},
  year    = {2019}
}

@inproceedings{schaul2016per,
  author    = {Schaul, Tom and Quan, John and Antonoglou, Ioannis and Silver, David},
  title     = {Prioritized Experience Replay},
  booktitle = {International Conference on Learning Representations},
  year      = {2016}
}

@article{shao2024deepseekmath,
  author  = {Shao, Zhihong and Wang, Peiyi and Zhu, Qihao and Xu, Runxin and Song, Junxiao and Zhang, Mingchuan and Li, Y.K. and Wu, Y. and Guo, Daya},
  title   = {{DeepSeekMath}: Pushing the limits of mathematical reasoning in open language models},
  journal = {arXiv preprint arXiv:2402.03300},
  year    = {2024}
}

@article{williams1992simple,
  author    = {Williams, Ronald J},
  title     = {Simple statistical gradient-following algorithms for connectionist reinforcement learning},
  journal   = {Machine learning},
  publisher = {Springer},
  volume    = {8},
  number    = {3},
  pages     = {229--256},
  year      = {1992}
}

@article{schulman2017proximal,
  author  = {Schulman, John and Wolski, Filip and Dhariwal, Prafulla and Radford, Alec and Klimov, Oleg},
  title   = {Proximal Policy Optimization Algorithms},
  journal = {arXiv preprint arXiv:1707.06347},
  year    = {2017}
}

@inproceedings{schulman2015trust,
  author    = {Schulman, John and Levine, Sergey and Moritz, Philipp and Jordan, Michael and Abbeel, Pieter},
  title     = {Trust Region Policy Optimization},
  booktitle = {Proc. of ICML},
  year      = {2015}
}

@inproceedings{horgan2018distributed,
  author    = {Daniel Horgan and John Quan and David Budden and Gabriel Barth-Maron and Matteo Hessel and Hado Van Hasselt and David Silver},
  title     = {Distributed Prioritized Experience Replay},
  booktitle = {International Conference on Learning Representations},
  year      = {2018}
}

@inproceedings{bengio2009curriculum,
  author       = {Bengio, Yoshua and Louradour, J{\'e}r{\^o}me and Collobert, Ronan and Weston, Jason},
  title        = {Curriculum learning},
  booktitle    = {Proceedings of the 26th annual international conference on machine learning},
  pages        = {41--48},
  year         = {2009},
  organization = {ACM}
}

@article{graves2017automated,
  author  = {Graves, Alex and Bellemare, Marc G and Menick, Jacob and Munos, Remi and Kavukcuoglu, Koray},
  title   = {Automated Curriculum Learning for Neural Networks},
  journal = {arXiv preprint arXiv:1704.03003},
  year    = {2017}
}

@inproceedings{ouyang2022training,
  author    = {Ouyang, Long and Wu, Jeff and Jiang, Xu and Almeida, Diogo and Wainwright, Carroll L and Mishkin, Pamela and Zhang, Chong and Agarwal, Sandhini and Slama, Katarina and Ray, Alex and others},
  title     = {Training language models to follow instructions with human feedback},
  booktitle = {Advances in Neural Information Processing Systems},
  volume    = {35},
  pages     = {27730--27744},
  year      = {2022}
}

@article{ziegler2019fine,
  author  = {Ziegler, Daniel M and Stiennon, Nisan and Wu, Jeffrey and Brown, Tom B and Radford, Alec and Amodei, Dario and Christiano, Paul and Irving, Geoffrey},
  title   = {Fine-tuning language models from human preferences},
  journal = {arXiv preprint arXiv:1909.08593},
  year    = {2019}
}

@misc{osband2026delightfulexploration,
  title         = {Delightful Exploration},
  author        = {Ian Osband},
  year          = {2026},
  eprint        = {2605.13287},
  archiveprefix = {arXiv},
  primaryclass  = {cs.LG},
  url           = {https://arxiv.org/abs/2605.13287}
}

@article{kaplan2020scaling,
  author  = {Kaplan, Jared and McCandlish, Sam and Henighan, Tom and Brown, Tom B. and Chess, Benjamin and Child, Rewon and Gray, Scott and Radford, Alec and Wu, Jeffrey and Amodei, Dario},
  title   = {Scaling laws for neural language models},
  journal = {arXiv preprint arXiv:2001.08361},
  year    = {2020}
}

@inproceedings{narayanan2021efficient,
  author    = {Narayanan, Deepak and Shoeybi, Mohammad and Casper, Jared and LeGresley, Patrick and Patwary, Mostofa and Korthikanti, Vijay and Vainbrand, Dmitri and Kashinkunti, Prethvi and Bernauer, Julie and Catanzaro, Bryan and Phanishayee, Amar and Zaharia, Matei},
  title     = {Efficient large-scale language model training on {GPU} clusters using {Megatron-LM}},
  booktitle = {Proceedings of the International Conference for High Performance Computing, Networking, Storage and Analysis},
  year      = {2021}
}

@article{lecun1998gradient,
  author  = {LeCun, Yann and Bottou, L{\'e}on and Bengio, Yoshua and Haffner, Patrick},
  title   = {Gradient-based learning applied to document recognition},
  journal = {Proceedings of the IEEE},
  volume  = {86},
  number  = {11},
  pages   = {2278--2324},
  year    = {1998}
}

@article{jiang2019accelerating,
  author  = {Jiang, Angela H. and Wong, Daniel L.-K. and Zhou, Giulio and Andersen, David G. and Dean, Jeffrey and Ganger, Gregory R. and Joshi, Gauri and Kaminsky, Michael and Kozuch, Michael and Lipton, Zachary C. and Pillai, Padmanabhan},
  title   = {Accelerating deep learning by focusing on the biggest losers},
  journal = {arXiv preprint arXiv:1910.00762},
  year    = {2019}
}

@article{auer2002finite,
  author  = {Auer, Peter and Cesa-Bianchi, Nicol{\`o} and Fischer, Paul},
  title   = {Finite-time analysis of the multiarmed bandit problem},
  journal = {Machine Learning},
  volume  = {47},
  number  = {2--3},
  pages   = {235--256},
  year    = {2002}
}
